\documentclass{article}

\usepackage{hyperref}

\usepackage[accepted]{icml2019}

\icmltitlerunning{Learning to Collaborate in Markov Decision Processes}

\usepackage{savesym,multirow,url,xspace,graphicx,hyperref,enumitem,color}
\usepackage[titletoc]{appendix}
\usepackage{subcaption}
\usepackage{dsfont}
\usepackage{multicol}

\usepackage{etoolbox}
\makeatletter
\newif\if@restonecol
\makeatother

\usepackage[lined, boxed, ruled, commentsnumbered, algo2e]{algorithm2e}
\usepackage{amsmath,amssymb,xfrac,amsthm}
\setitemize{noitemsep,topsep=1pt,parsep=1pt,partopsep=1pt, leftmargin=12pt}
\newtheorem{lemma}{Lemma}
\newtheorem{theorem}{Theorem}
\newtheorem*{theorem*}{Theorem}

\newtheorem{corollary}{Corollary}
\newtheorem{proposition}{Proposition}
\newtheorem{definition}{Definition}
\makeatletter

\newcommand{\Rmnum}[1]{\expandafter\@slowromancap\romannumeral #1@}
\makeatother
\usepackage{caption}

\newcommand{\algexprestart}{\textsc{ExpRestart}}
\newcommand{\algexpsmooth}{\textsc{ExpDRBias}}

\usepackage{enumitem}
\usepackage{wasysym}

\DeclareMathOperator*{\argmax}{arg\,max}

\newcommand{\agentone}{agent $\mathcal A_1$}
\newcommand{\agenttwo}{agent $\mathcal A_2$}

\newcommand{\woagentone}{$\mathcal A_1$}
\newcommand{\woagenttwo}{$\mathcal A_2$}

\newcommand{\mdp}{MDP}

\newcommand{\states}{S}
\newcommand{\actions}{A}
\newcommand{\actionsone}{A_1}
\newcommand{\actionstwo}{A_2}
\newcommand{\action}{a}
\newcommand{\actionone}{{a^1}}

\newcommand{\actiontwo}{{a^2}}

\newcommand{\state}{s}
\newcommand{\rew}{r}
\newcommand{\rews}{\bm r}
\newcommand{\rewst}{{\bm r_t}}
\newcommand{\avgrew}{\eta}
\newcommand{\avgrewt}{{\eta_t}}
\newcommand{\return}{V}
\newcommand{\influence}{I}
\newcommand{\influenceone}{I_1}
\newcommand{\influencetwo}{I_2}

\newcommand{\statdist}{\bm d}
\newcommand{\statdistr}{{\bm d_{t,m}}}

\newcommand{\diampoli}{\Delta_{\pi_i}}
\newcommand{\diampolone}{\Delta_{\pi_1}}
\newcommand{\diampoltwo}{\Delta_{\pi_2}}

\newcommand{\relrangei}{\kappa_i}
\newcommand{\relrangeone}{\kappa_1}
\newcommand{\relrangetwo}{\kappa_2}

\newcommand{\statdists}{{\bm d_{{\bm \pi^1},{\bm  \pi^2}}}}
\newcommand{\statdistt}{{\bm d_{{\bm \pi^1}_t,{\bm \pi^2}_t}}}
\newcommand{\trans}{\bm P}
\newcommand{\transt}{{\bm P_{{\bm \pi^1}_t,{\bm \pi^2}_t}}}
\newcommand{\pol}{\bm \pi}

\newcommand{\polone}{{\bm \pi^1}}
\newcommand{\polonet}{{{\bm \pi^1}_{t}}}
\newcommand{\poltwo}{{\bm \pi^2}}
\newcommand{\poltwot}{{{\bm \pi^2}_{t}}}

\newcommand{\qval}{Q}

\newcommand{\qvals}{\mathbf Q}
\newcommand{\qvalst}{{\mathbf Q_t}}
\newcommand{\regret}{R}

\newcommand{\ratechange}{\rho}
\newcommand{\ratechangeone}{{\rho_1}}
\newcommand{\ratechangetwo}{{\rho_2}}
\newcommand{\responsivity}{\alpha}

\newcommand{\smoothone}{\lambda}
\newcommand{\smoothtwo}{\mu}

\newcommand{\weights}{\mathbf w}

\newcommand{\regul}{\mathcal R}

\newcommand{\simplexone}{{\mathcal P_{A_1}}}

\newcommand{\tixone}{t}
\newcommand{\tixtwo}{\tau}
\newcommand{\tix}{T}
\newcommand{\rounds}{M}
\newcommand{\round}{m}

\newcommand{\mixtime}{\omega}

\newcommand{\Segment}{\Gamma}
\newcommand{\ratelearn}{\epsilon}

\newcommand{\transpose}{\dagger}

\newcommand{\E}{\mathds E}

\newcommand{\opt}{\textsc{{Opt}}\xspace}

\newcommand{\bigObound}{\ensuremath{\mathcal{O}}}

\newcommand{\Thetabound}{\ensuremath{\Theta}}
\newcommand{\Omegabound}{\ensuremath{\Omega}}

\usepackage{bm}
\usepackage{bbm}
\let\mathbbm\mathds
\newcommand{\norm}[1]{\left\lVert#1\right\rVert}

\newcommand{\normL}[1]{\left\lVert#1\right .}
\newcommand{\normR}[1]{\left . #1\right\rVert}

\newcommand{\vecdot}[1]{\left<#1\right>}

\newcommand{\bracket}[1]{\left[#1\right]}

\newcommand{\floor}[1]{\lfloor #1 \rfloor}

\begin{document}

\twocolumn[
\icmltitle{Learning to Collaborate in Markov Decision Processes}



\icmlsetsymbol{equal}{*}
\begin{icmlauthorlist}
  \icmlauthor{Goran Radanovic}{af1}
  \icmlauthor{Rati Devidze}{af2}
  \icmlauthor{David C. Parkes}{af1}
  \icmlauthor{Adish Singla}{af2}
\end{icmlauthorlist}
\icmlaffiliation{af1}{Harvard University.}
\icmlaffiliation{af2}{Max Planck Institute for Software Systems (MPI-SWS)}
\icmlcorrespondingauthor{Goran Radanovic}{gradanovic@g.harvard.edu}
\icmlcorrespondingauthor{Adish Singla}{adishs@mpi-sws.org}
\icmlkeywords{Machine Learning, ICML}
\vskip 0.3in
]



\printAffiliationsAndNotice{}  

\newtoggle{longversion}
\settoggle{longversion}{true}
\begin{abstract}
We consider a two-agent MDP framework where agents repeatedly solve a task in a collaborative setting. We study the problem of designing a learning algorithm for the first agent (\woagentone) that facilitates  successful collaboration even in cases when the second agent (\woagenttwo) is adapting its policy in an unknown way. 
The key challenge in our setting is that the first agent faces non-stationarity in rewards and transitions because of the adaptive behavior of the second agent.

\looseness-1
We design novel online learning algorithms for \agentone~whose regret decays as $\bigObound\Big(\tix^{\max\{1-\frac{3}{7}\cdot \alpha, \frac{1}{4}\}}\Big)$, for $\tix$ learning episodes, provided that the magnitude in the change
in \agenttwo's policy between any two consecutive episodes is upper bounded by $\bigObound(\tix^{-\responsivity})$. Here, the parameter $\responsivity$ is assumed to be strictly greater than $0$, and we show that this assumption is necessary provided that the {\em learning parity with noise} problem is computationally hard. We show that sub-linear regret of \agentone~further implies near-optimality of the agents' joint return for MDPs that manifest the properties of a {\em smooth} game. 
\end{abstract}
\section{Introduction}\label{sec.introduction}

Recent advancements in AI have the potential to change our daily lives by boosting productivity (e.g., via virtual personal assistants), augmenting human capabilities (e.g., via smart mobility systems), and increasing automation (e.g., via auto-pilots and assistive robots). These are settings of intelligence augmentation, where societal benefit will come 
%
not from complete automation but rather from the interaction between people and machines, in a process of a productive human-machine collaboration.


We expect that useful collaboration will come about through AI agents that can adapt to the behavior of users. 
As an example, consider self-driving cars where auto-pilots can be overridden by human drivers. In an initial period, 
a human driver would likely change their behavior until accustomed with 
new features of an auto-pilot mode. 
%
%
Without accounting for this changing behavior of users, the performance of the AI agent could considerably deteriorate, leading to, for example, hazardous situations in an auto-pilot mode. 
Hence, it is important that the AI agent updates its decision-making policy accordingly.

We formalize this problem through a two-agent, reinforcement learning (RL) framework.
%
The agents, hereafter referred to as \agentone~and \agenttwo, jointly solve a task in a collaborative setting (i.e., share a common reward function and a transition kernel that is based on their joint actions). Our goal is to develop a learning algorithm for \agentone~that facilitates a successful collaboration even in cases when \agenttwo~is adapting its own policy.
In the above examples, \agentone~could represent the AI agent whereas \agenttwo~could be a person with time-evolving behavior.
We primarily focus on an \emph{episodic} Markov decision process (MDP) setting, in which the agents repeatedly interact: 
\vspace{-7mm}
\begin{enumerate}[label=(\roman*)]
    \item \agentone~decides on its policy based on historic information (\agenttwo's past policies) and the underlying MDP model;
    \vspace{-2mm}
    \item \agentone~commits to its policy 
for a given episode without knowing the policy of \agenttwo;
    \vspace{-2mm}
    \item \agentone~updates its policy at the end of the episode based on \agenttwo's observed behavior.
\end{enumerate}
\vspace{-2mm}
%
%
When \agenttwo's policy is fixed and known, one can find 
an optimal policy for \agentone~using standard MDP planning techniques. 
In our setting, however, we do not assume \agenttwo's behavior is stationary,
and we do not  assume any particular model for how \agenttwo~changes its policy. 
This differs from similar two-agent (human-AI) collaborative settings \cite{dimitrakakis2017multi,nikolaidis2017game} that prescribe a particular behavioral model to \agenttwo~(human agent).

\subsection{Overview of Our Approach}
\looseness-1

The presence of \agenttwo~in our framework implies that the reward function and transition kernel are non-stationary from the perspective of \agentone.
Variants of the setting have also been studied in the learning literature~\cite{even2005experts,even2009online,yu2009online,yu2009arbitrarily,yu2009markov,abbasi2013online,wei2017online}. However, these  approaches do not directly apply 
because: (i) they focus on a particular aspect of non-stationarity (e.g., changing rewards with fixed transitions) \cite{even2005experts,even2009online}, (ii) require that the changes in the transition model are bounded \cite{yu2009arbitrarily,yu2009online}, (iii)  make restrictions on the policy space \cite{abbasi2013online}, and (iv) consider a competitive or adversarial setting instead of cooperative setting with shared reward~\cite{wei2017online}. 
%
%
%
%
%
Instead, we will assume that \agenttwo~does not abruptly change its policy across episodes,
and prove that the problem becomes computationally intractable otherwise. Our approach is inspired by the problem of experts learning in MDPs~\cite{even2005experts}, in which each state is associated with an experts algorithm that derives the policy for that state using $\qval$-values. 
However, to compensate for the non-stationarity of transitions and facilitate a faster learning process, we introduce novel forms of recency bias inspired by the ideas of \citet{rakhlin2013optimization} and \citet{syrgkanis2015fast}.





\looseness-1
\textbf{Contributions.} We design novel algorithms for \agentone~that lead to sub-linear regret ~of $\bigObound(\tix^{\max\{1-\frac{3}{7}\cdot \alpha, \frac{1}{4}\}})$, where $\tix$ is the number of episodes.  
 We assume the magnitude 
 of \agenttwo's policy change w.r.t. $\tix$ is  $\bigObound(\tix^{-\responsivity})$, for parameter $\responsivity$ that we require to be strictly positive.
%
 We show via a reduction from the {\em learning parities with noise} problem \cite{abbasi2013online,kanade2014learning}
that this upper bound on the rate of change in \agenttwo's policy is necessary, in that it is  computationally hard to achieve sub-linear regret for the special case of $\responsivity = 0$.
 Furthermore, 
we connect the agents' joint return to the regret of \agentone~by adapting the concept of {\em smoothness} from the game-theory literature \cite{roughgarden2009intrinsic,syrgkanis2015fast}, and  
we show that the bound on the regret of \agentone~implies near optimality of the agents' joint return for \mdp s that manifest a {\em smooth} game \cite{roughgarden2009intrinsic,syrgkanis2015fast}. To the best of our knowledge, we are the first to provide such guarantees in a collaborative two-agent MDP learning setup.   
\iftoggle{longversion}{}{The proofs can be found in the extended version of the paper \cite{radanovic2019learning}.}

\vspace{-2mm}
\section{The Setting}\label{sec.model}
We model a two-agent learning problem through an \mdp~environment.\footnote{An \mdp~ with multiple agents is often called {\em multiagent} \mdp~\cite{boutilier1996planning}.} The agents are \agentone~and \agenttwo.  We consider an episodic setting with $\tix$  episodes (also called time steps) and each episode lasting $\rounds$ rounds. Generic episodes are denoted by $\tixone$ and $\tixtwo$, while a generic round is denoted by $\round$.
The \mdp~is defined by:
\vspace{-2mm}
\begin{itemize} 
\item a finite set of states $\states$, with $\state$ denoting a generic state. We enumerate the states by $1$, ..., $|\states|$, and assume this ordering in our vector notation.
\item a finite set of actions $\actions = \actionsone \times \actionstwo$, with $\actionone \in \actionsone$ denoting a generic action of \agentone~and  $\actiontwo \in \actionstwo$ denoting a generic action of \agenttwo. 
%
We enumerate the actions of \agentone~by $1$, ..., $|\actions_{1}|$ and \agenttwo~by $1$, ..., $|\actions_{2}|$, and assume this ordering in our vector notation.
\item a transition kernel $\trans(\state, \actionone, \actiontwo, \state_{new})$, which  is a tensor with indices defined by the current state, the agents' actions, and the next state.
\item a reward function $\rew:  \states \times \actions \rightarrow [0, 1]$ that defines the joint reward for both agents.
\end{itemize}

We assume that \agentone~knows the MDP model. The agents {\em commit} to playing {\em stationary} policies $\polonet$ and $\poltwot$ in each episode $\tixone$, but do so without knowing the commitment of the other agent. At the end of the episode $\tixone$, the agents observe each other's policies ($\polonet$, $\poltwot$) and can use this information to update their future policies.\footnote{We focus on the full information setting as it allows us to do a cleaner analysis while revealing some of the
challenges of the problem at hand. The setting, for example, formalizes a scenario where the AI describes its policy to the human, and the episodes are long enough that the AI can effectively observe the human's policy.}
%
%
Since the state and action spaces are finite, policies can be represented as matrices $\polone_{\tixone}(\state, \actionone)$ and $\poltwo_{\tixone}(\state, \actiontwo)$, so that rows  $\polone_{\tixone}(\state)$
and $\poltwot(\state)$ define distributions on actions in a given state.
%
We also define the reward matrix for \agentone~as $\rewst( \state, \actionone) = \E_{\actiontwo \sim \poltwot(\state)} \bracket{ \rew( \state, \actionone, \actiontwo) }$, whose elements are the expected 
rewards of \agentone~for different actions and states. By bounded
rewards, we have $0 \le \rewst( \state, \actionone) \le 1$.


\subsection{Objective}

After each episode $\tixone$, the agents can adapt their policies. However,  \agenttwo~is not in 
our control, and not assumed to be optimal. Therefore, we take the perspective of \agentone, and seek to optimize its policy in order to obtain good
joint returns. The joint return in episode $\tixone$ is:
\begin{align*}
&\return_{\tixone} =  \frac{1}{M} \cdot \E \bracket{ \sum_{\round = 1}^\rounds \rew(\state_{\round}, \actionone_{\round}, \actiontwo_{\round}) | \statdist_1, \polonet, \poltwot}\\
&= \frac{1}{M} \cdot \sum_{\round = 1}^{\rounds} \statdistr \cdot \vecdot{ \polonet, \rewst},
\end{align*}

where $\state_{\round}$ is the state at round $\round$.
For $\round = 1$, $\state_{\round}$ is sampled from 
the initial state distribution $\statdist_1$.
For later periods, $\state_{\round}$ is obtained by  
following joint actions $(\actionone_{1}, \actiontwo_{1})$, $(\actionone_{2}, \actiontwo_{2})$, ..., $(\actionone_{\round - 1}, \actiontwo_{\round - 1})$ from state $\state_1$. Actions are
obtained from policies $\polonet$ and  $\poltwot$. The second part of the equation 
uses a vector notation to define  the joint return, where $\statdistr$ is a row vector representing the state distribution at episode $\tixone$ and round $\round$, while $\vecdot{\polonet, \rewst}$ is a row-wise dot product whose result is a column vector with $|\states|$ elements.
 Since this is an episodic framework, we will assume the same starting state distribution, $\statdist_1$, 
for all episodes $\tixone$. However $\statdistr$ can differ across episodes since policies $\polonet$ and 
$\poltwot$ evolve. 

We define the average return over all episodes as $\bar V = \frac{1}{\tix} \cdot \sum_{\tixone = 1}^{\tix} \return_{\tixone}$. The 
 objective is to output a sequence of \agentone's policies $\polone_1$, ..., $\polone_\tix$ that maximize:
\begin{align*}
\sup_{\polone_1, ..., \polone_\tix} \bar V = \sup_{\polone_1, ..., \polone_\tix} \frac{1}{\tix} \cdot \sum_{\tixone = 1}^{\tix} \return_{\tixone}.
\end{align*}

The maximum possible value of $\bar V$ over all combinations of \agentone's and \agenttwo's policies is denoted as $\opt$. Notice that this value is achievable using MDP planning techniques, provided that we control both agents. 



\subsection{Policy Change Magnitude and Influences}


We do not control \agenttwo, and we do note  assume that  \agenttwo~follows a particular behavioral model. Rather,
we quantify the allowed behavior via the {\em policy change magnitude}, which for \agenttwo~is defined as: 
\begin{align*}
\ratechangetwo &= \max_{\tixone > 1, \state} \sum_{\actiontwo \in \actionstwo} |\poltwo_{\tixone}(\state, \actiontwo) - \poltwo_{\tixone-1}(\state, \actiontwo)| \\
&= \max_{\tixone > 1}  \norm{\poltwo_{\tixone} - \poltwo_{\tixone-1}}_{\infty},
\end{align*}
%
where $\norm{\cdot}_{\infty}$ is operator (induced) norm.
%
%
In the case of \agenttwo, we will be focusing on policy change magnitudes $\ratechangetwo$ that are of the order $\bigObound(\tix^{-\responsivity})$, where $\responsivity$ is strictly grater than $0$. For instance, the assumption holds if \agenttwo~is a learning agent that adopts  
the experts in MDP approach of \citet{even2005experts,even2009online}.



We also define  the {\em influence} of an agent on the transition dynamics. This measures how much an agent can influence the transition dynamics through
its policy. 
For \agenttwo, the influence is defined as:
\begin{align*}
&\influencetwo = \sup_{\polone, \poltwo \ne \poltwo'} \frac{\norm{\trans_{\polone,\poltwo} - \trans_{\polone,\poltwo'}}_{\infty}}{\norm{\poltwo - \poltwo'}_{\infty}},
\end{align*}
where kernel (matrix) $\trans_{\polone,\poltwo}(\state, \state_{new})$ denotes the probability of transitioning from $\state$ to $\state_{new}$ when the agents' policies are  $\polone$ and  $\poltwo$ respectively.\footnote{Our notion of influence is similar to, although not the same as, that of \citet{dimitrakakis2017multi}.} 
%
Influence is a  measure of how much an agent affects the transition probabilities by changing its  policy. We are primarily interested in using this notion to show how our approach compares to the existing results from the online learning  literature. For $\influencetwo = 0$, our setting relates to the single agent settings of \citet{even2005experts,even2009online,dick2014online} where rewards are non-stationary but transition probabilities are fixed. In general, the  influence $\influence_2$ takes values in $[0, 1]$
\iftoggle{longversion}{(see Appendix~\ref{appendix.useful_properties}, Corollary~\ref{cor_influence_bound})}{(see Appendix~B (Corollary~1) of \citet{radanovic2019learning})}.
We  can analogously define policy change magnitude $\ratechangeone$, and influence $\influenceone$ of \agentone. 



\subsection{Mixing Time and $\qval$-values}\label{sec.mixtime_qvals}

We follow standard assumptions from the literature on  online learning in MDPs (e.g., see \citet{even2005experts}), and  only consider transition kernels that have well-defined stationary distributions. 
For the associated transition kernel, 
we define a stationary state distribution $\statdists$ as the one for which:
\vspace{-2mm}
\begin{enumerate}
\item any initial state distribution converges to under policies $\polone$ and  $\poltwo$;
\vspace{-1mm}
\item and  $\statdists \cdot \trans_{\polone,\poltwo} = \statdists$.
\end{enumerate}
\vspace{-2mm}
Note that $\statdists$ is represented as a row vector with $|\states|$ elements.
Furthermore, as discussed in \citet{even2005experts}, this implies that there exists a mixing time $\mixtime$, such that for all state distributions $\statdist$ and $\statdist'$, we have 
\begin{align*}
\norm{\statdist \cdot \trans_{\polone,\poltwo} - \statdist' \cdot \trans_{\polone,\poltwo}}_1 \le e^{-\frac{1}{\mixtime}} \cdot  \norm{\statdist - \statdist'}_1.
\end{align*}
Due to this well-defined mixing time, we can define the {\em average reward} of \agentone~when following policy $\polone$ in episode $\tixone$ as:
\begin{align*}
\avgrewt(\polone) := \avgrew_{\poltwot}(\polone) = \statdist_{\polone,\poltwot} \cdot \vecdot{ \polone, \rewst},
\end{align*}
where $\vecdot{., .}$ is row-wise dot product whose result is a column vector with $|\states|$ elements. 
The $\qval$-value matrix for \agentone~w.r.t. policy $\polone_{\tixone}$ is defined as:
\begin{align*}
&\qvalst(\state, \actionone) = \E \bracket{ \sum_{\round = 1}^\infty \big(\rewst(\state_{\round}, \actionone_{\round}) - \avgrew_{\tixone}(\polone_{\tixone})\big) | \state, \actionone, \polonet},
\end{align*} 
where $\state_\round$ and $\actionone_\round$ are states and actions in round $\round$, starting from state $\state$ with action $\actionone$ and then using policy $\polonet$.
Moreover, the policy-wise $Q$-value (column) vector for $\polone$ w.r.t. policy $\polone_{\tixone}$ is defined by:
\begin{align*}
\qvals_{\tixone}^\polone(\state) = \E_{\actionone \sim \polone(\state)}\bracket{\qvalst(\state, \actionone)},
\end{align*}
%
and in matrix notation $\qvals_{\tixone}^\polone = \vecdot{\polone, \qvals_{\tixone}}$.
The $\qval$-values satisfy the following Bellman equation:
\begin{align*}
\qvals_{\tixone}(\state, \actionone) = \rews_t(\state, \actionone) - \avgrew_{\tixone}(\polone_{\tixone}) + \trans_{\poltwo_{\tixone}}(\state, \actionone) \cdot \qvals_{\tixone}^\polonet, 
\end{align*}
where $\trans_{\poltwo_{\tixone}}(\state, \actionone)$ defines the probability distribution over next states given action $\actionone$ of \agentone~and policy $\poltwo_{\tixone}$ of \agenttwo~(here, $\trans_{\poltwo_{\tixone}}(\state, \actionone)$ is denoted as a row vector with $|\states|$ elements).
For other useful properties of this \mdp~ framework we refer the reader to \iftoggle{longversion}{Appendix~\ref{appendix.useful_properties}}{Appendix~B of \citet{radanovic2019learning}}.

\section{Smoothness and No-regret Dynamics}\label{sec.approach}

The goal is to output a sequence of \agentone's policies $\polone_{1}$, ..., $\polone_{\tix}$ so that the joint return $\bar \return$ is maximized. There are two key challenges: (i) \agenttwo~policies could be sub-optimal (or, even adversarial in the extreme case), and (ii) \agentone~does not know the current policy of \agenttwo~at the beginning of episode $\tixone$. 

\textbf{Smoothness Criterion.}
To deal with the first challenge, we consider a structural assumption that enables us to apply a regret analysis when quantifying the quality of a solution w.r.t. the optimum.  In particular, we assume that the \mdp~ is {\em ($\smoothone,\smoothtwo$)-smooth}:
\begin{definition}\label{def_smoothness}
We say that an \mdp~ is {\em ($\smoothone,\smoothtwo$)-smooth} if there exists a pair of policies $(\polone^*, \poltwo^*)$ 
such that for every policy pair $(\polone, \poltwo)$:
\begin{align*}
\avgrew_{\poltwo}(\polone^*) &\ge \smoothone \cdot \avgrew_{\poltwo^*}(\polone^*) - \smoothtwo \cdot \avgrew_{\poltwo}(\polone),\\
\avgrew_{\poltwo^*}(\polone^*) &\ge \avgrew_{\poltwo}(\polone).
\end{align*}
\end{definition}

This bounds the impact of \agenttwo's policy on the average reward. In particular,
there must exist 
an \textit{optimal} policy pair $(\polone^*, \poltwo^*)$ such that the negative impact 
of \agenttwo~for choosing $\poltwo \ne \poltwo^*$ is controllable by an appropriate choice of \agentone's policy.
This definition is a variant of the 
 {\em smoothness} notion introduced to study  the ``price-of-anarchy''
 of non-cooperative games, including for learning dynamics~\cite{roughgarden2009intrinsic,syrgkanis2015fast}. 
For the relationship between the smoothness parameters and the properties of the MDP, we refer the reader to 
\iftoggle{longversion}{Appendix~\ref{appendix.connection_to_smoothenss}}{Appendix~C of \citet{radanovic2019learning}}.
It is important to note that since we have a finite number of rounds
$\rounds$ per episode, $\opt$ is not necessarily the same as
$\avgrew_{\poltwo^*}(\polone^*)$, and the policies that achieve $\opt$
need not lead to $\avgrew_{\poltwo^*}(\polone^*) $.

\textbf{No-regret Learning.}
To address the second challenge, we adopt the online learning framework and seek to minimize \textit{regret} $\regret(T)$:
\begin{align}\label{eq_regret_def}
\regret(\tix) = \sup_{\polone} \sum_{\tixone = 1}^\tix \bracket{ \avgrewt(\polone) -  \avgrewt(\polonet) }.
\end{align} 

A policy sequence $\polone_1$, ..., $\polone_\tix$
is {\em no-regret} if regret $\regret(\tix)$ is sublinear in $\tix$.
%
An algorithm that outputs such sequences is a {\em no-regret algorithm} --- this intuitively means that the agent's performance is competitive w.r.t. any fixed policy.

\textbf{Near-optimality of No-regret Dynamics.}
Because \agenttwo~could be adapting to the policies of \agentone, this is an adaptive learning setting, and the notion of regret can become less useful. This is where the smoothness criterion comes in. We will show that it suffices to minimize the regret $\regret(T)$ in order to obtain near-optimal performance. 

Using an analysis similar to \citet{syrgkanis2015fast}, we establish the near-optimality of no-regret dynamics defined w.r.t.~the optimal return $\opt$, as stated in the following lemma:




\begin{lemma}\label{lemma_large_rounds}
For a problem with ($\smoothone,\smoothtwo$)-smooth MDP, return $\bar \return$ is lower bounded by:
\begin{align*}
\bar \return &\ge \frac{\smoothone}{1+ \smoothtwo} \cdot \opt - \frac{1}{1+\smoothtwo} \cdot \frac{\regret(\tix)}{\tix} -  2\cdot \frac{1 + \frac{\smoothone}{1+ \smoothtwo}}{M\cdot (1- e^{-\frac{1}{\mixtime}})}.
\end{align*} 
\end{lemma}
\iftoggle{longversion}{
\begin{proof}
See \iftoggle{longversion}{Appendix~\ref{appendix.proof_lemma_large_rounds}}{Appendix~D in \cite{radanovic2019learning}} for the proof.
\end{proof}
}


Lemma \ref{lemma_large_rounds} implies that as the number of episodes
$\tix$ and the number of rounds $M$ go to infinity, return $\bar
\return$ converges to a multiple $\frac{\smoothone}{1+ \smoothtwo}$ of
the optimum $\opt$, provided that \agentone~is a no-regret learner. In the next section, we design such no-regret learning algorithms for \agentone. 
\section{Learning Algorithms}\label{sec.mechanism}

We base our approach on the expert learning literature for \mdp s, in particular
that of~\citet{even2005experts,even2009online}. 
The basic idea is to associate each state with an experts algorithm,
and decide on a policy by examining the $Q$-values of state-action pairs. Thus, the $Q$ function  represents a reward function in the expert terminology. 

\subsection{Experts with Periodic Restarts: \algexprestart}\label{sec.experts_restarts}

In cases when \agenttwo~ has no influence on transitions, the approach of \citet{even2005experts,even2009online} would yield the no-regret guarantee.
The main difficulty of  the present setting is that \agenttwo~can influence the transitions via its policy. The hope is that as long as the magnitude of policy change by  \agenttwo~across episodes is not too large, \agentone~can compensate for the non-stationarity by using only recent history when updating its policy. 

A simple way of implementing this principle is to use a no-regret learning algorithm, but periodically restarting it, i.e., by splitting the full time horizon into segments of length $\Segment$, and applying the algorithm on each segment separately. In this way, we have well-defined periods $\{1, ..., \Segment\}$, $\{\Segment + 1, ..., 2 \cdot \Segment\}$, ...,  $\{\tix - \Segment + 1, ..., \tix\}$. 
As a choice of an expert algorithm (the algorithm 
associated with each state), we 
use {\em Optimistic Follow the Regularized Leader} (OFTRL) \cite{rakhlin2013optimization,syrgkanis2015fast}. 
Our policy updating rule for segment $l$, with starting point $\tixtwo = 1+(l-1)\cdot \Segment$, can be described as:
\begin{align*}
\polone_{\tixone}(\state) = \argmax_{\weights \in \simplexone}{ \left (\sum_{k = \tixtwo}^{\tixone-1}\qvals_{k}(\state) + \qvals_{\tixone-1}(\state) \right ) \weights^{\transpose}  + \frac{\regul(\weights)}{\ratelearn}}
\end{align*}
for $\tixone \in \{\tixtwo, \ldots, l \cdot \Gamma\}$, and:
\begin{align*}
 \polone_{\tixone}(\state) = \argmax_{\weights \in \simplexone}{\frac{\regul(\weights)}{\ratelearn}} \mbox{ for } \tixone = \tixtwo.
\end{align*}

$\qvals_{k}(\state)$ denotes a row of matrix $\qvals_{k}$ (see Section \ref{sec.mixtime_qvals}),\footnote{Given $\polonet$ and $\poltwot$, we can calculate $\qvalst$ from the Bellman equation using standard dynamic programming techniques.} $\weights$ is a row vector from probability simplex $\simplexone$, $\transpose$ denotes the transpose operator, $\regul$ is a 1-strongly convex regularizer w.r.t.~norm $\norm{\cdot}_1$, and $\ratelearn$ is the learning rate. 
This approach, henceforth referred to as {\em experts with periodic restarts} (\algexprestart),  suffices to obtain  sublinear regret provided that the segment length $\Segment$ and learning rate $\ratelearn$ are properly set (see \iftoggle{longversion}{Appendix~\ref{appendix.properties_algo_oftrl_rand_restart}}{Appendix~G of \citet{radanovic2019learning}}). 

One of the main drawbacks of experts with periodic restarts is that it potentially results in abrupt changes in the policy of \agentone, this occurring when switching from one segment to another. In practice, one might want to avoid this, for example, because \agenttwo~(e.g., representing a person) might negatively respond to such abrupt changes in \agentone's policy. Considering this, we design a new experts algorithm
that ensures gradual policy changes for \agentone~across episodes, while achieving the same order of regret guarantees (see Section~\ref{sec.subsection4-regret-analysis} and \iftoggle{longversion}{Appendix~\ref{appendix.properties_algo_oftrl_rand_restart}}{Appendix~G of \citet{radanovic2019learning}}).





	


\subsection{Experts with Double Recency Bias: \algexpsmooth}

Utilizing fixed segments, as in the approach of \algexprestart, leads to potentially rapid policy changes after each segment. 
To avoid this issue, we can for each episode $\tixone$ consider a family of segments of different lengths:
$\{\tixone - \Segment, ..., \tixone \}$, $\{\tixone - \Segment + 1,
..., \tixone \}$, ..., $\{\tixone - 1, \tixone\}$, 
and run the OFTRL algorithm on each segment separately. The policy in episode $\tixone$ can then be defined as the average of the OFTRL outputs. 
This approach, henceforth referred to as {\em experts with double recency bias} (\algexpsmooth), can be implemented through the following two ideas 
that bias the policy selection rule towards recent information in a twofold manner.

\textbf{Recency Windowing.}
The first idea is what we call {\em recency windowing}. Simply put, it specifies how far in the history an agent should look when choosing a policy.
More precisely, we define a sliding window of size $\Segment$ and to decide on policy $\polonet$ we only use historical information from periods after $\tixone - \Segment$. In particular, the updating rule of OFTRL would be modified for $\tixone > 1$ as $\polone_{\tixone}(\state)=$ 
\begin{align*}
&\argmax_{\weights \in \simplexone}{ \bigg (\sum_{k=\max(1, \tixone - \Segment)}^{\tixone-1}\qvals_{k}(\state) + \qvals_{\tixone-1}(\state) \bigg )}  \weights^{\transpose} + \frac{\regul(\weights)}{\epsilon}.
\end{align*}
and:
\begin{align}\label{eq_update_windowing_init}
 \polone_{1}(\state) = \argmax_{\weights \in \simplexone}{\frac{\regul(\weights)}{\ratelearn}} \mbox{ for } \tixone = 1.
\end{align}



\begin{algorithm2e}[t!]
\SetKwBlock{When}{when}{endwhen} 

\textbf{Input:} {History horizon $\Segment$, learning rate $\ratelearn$\\}
\Begin{
\textbf{Initialize:} $\forall s$, compute $\polone_{1}(s)$ using Eq. \eqref{eq_update_windowing_init}\\

\For{episode $\tixone \in \{ 1,..., \tix \}$}{
	$\forall s$, commit to policy $\polone_{\tixone}(s)$\\
	Obtain the return $\return_{\tixone}$\\
	Observe \agenttwo's policy $\poltwo_\tixone$\\
	Calculate Q-values $\qvals_\tixone$\\
		$\forall s$, compute $\polone_{\tixone+1}(s)$ using Eq. \eqref{eq_update_windowing_policy}\\
}
}
\caption{\algexpsmooth}\label{algo_leader_oftrl}
\end{algorithm2e}




\textbf{Recency Modulation.} The second idea is what we call {\em recency modulation}. This  creates an averaging effect over the policies computed by the experts with periodic restarts approach, for different possible starting points of the segmentation.
For episode $\tixone$, recency modulation calculates policy updates using recency windowing but considers windows of different sizes. 
More precisely, we calculate updates with window sizes $1$ to $\Segment$, and then average them to obtain the final update. 
Lemma~\ref{lm_weight_diff} shows that this updating rule will not lead to abrupt changes in \agentone's policy.

 To summarize, \agentone~has the following policy update rule for $\tixone > 1$: 
\vspace{-2mm}
\begin{align}\label{eq_update_windowing_policy}
\polone_{\tixone}(s) = \frac{1}{\Segment} \sum_{\tixtwo = 1}^{\Segment} \weights_{\tixone, \tixtwo}(s),
\vspace{-2mm}
\end{align}
where $\weights_{\tixone, \tixtwo}(\state)=$
\vspace{-2mm}
\begin{align*}
&\argmax_{\weights \in \simplexone}{ \bigg (\sum_{k=\max(1, \tixone - \tixtwo)}^{\tixone-1}\qvals_{k}(\state) + \qvals_{\tixone-1}(\state) \bigg )}  \weights^{\transpose} + \frac{\regul(\weights)}{\epsilon}.
\vspace{-2mm}
\end{align*}

For $t=1$, we follow equation update \eqref{eq_update_windowing_init}.
%
 The full description of \agentone's policy update using the approach of \algexpsmooth~is given in Algorithm~\ref{algo_leader_oftrl}.  As with \algexprestart, \algexpsmooth~leads to a sub-linear regret for a proper choice of $\ratelearn$ and $\Segment$, which in turn results in a near-optimal behavior, as analyzed in the next section.

\section{Theoretical Analysis of \algexpsmooth}\label{sec.analysis}


To bound regret $\regret(\tix)$, given by equation \eqref{eq_regret_def}, it is useful to express difference $\avgrewt(\polone) -  \avgrewt(\polonet)$ in terms of $\qval$-values. In particular, one can show that this difference is equal to $\statdist_{\polone, \poltwo_{\tixone}} \cdot ( \qvalst^\polone  - \qvalst^\polonet )$ (see \iftoggle{longversion}{Lemma~\ref{lm_eta_q_relation} in Appendix~\ref{appendix.useful_properties}}{Lemma~15 in Appendix~B  of \citet{radanovic2019learning}}). 
By the definitions of $\qvalst^\polone$ and $\qvalst^\polonet$, this implies: 
\begin{align*}
 \avgrewt(\polone) -  \avgrewt(\polonet) = \statdist_{\polone, \poltwo_{\tixone}} \cdot \vecdot{\polone - \polonet, \qvals_{\tixone}}. 
\end{align*}
If $\statdist_{\polone, \poltwo_{\tixone}}$ was not dependent on $\tixone$ (e.g., if \agenttwo~was not changing its policy), then bounding $\regret(\tix)$ would amount to bounding the sum of terms $\vecdot{\polone - \polonet, \qvals_{\tixone}}$. 
This could be done with an approach that carefully combines the proof techniques of \citet{even2005experts} with the OFTRL properties, in particular, {\em regret bounded by variation in utilities} (RVU) \cite{syrgkanis2015fast}. However, in our setting $\statdist_{\polone, \poltwo_{\tixone}}$ is generally changing with $\tixone$.

\subsection{Change Magnitudes of Stationary Distributions}

To account for this, we need to investigate how quickly distributions $\statdist_{\polone, \poltwo_{\tixone}}$ change across episodes. Furthermore, to utilize the RVU property, we need to do the same for distributions $\statdist_{\polone_{\tixone}, \poltwo_{\tixone}}$. The following lemma provides bounds on the respective change magnitudes.   

\begin{lemma}\label{lm_stat_dist_inequality}
The difference between the stationary distributions of two consecutive 
episodes is upper bounded by: 
\begin{align*}
\norm{\statdist_{\polone_{\tixone}, \poltwo_{\tixone}} - \statdist_{\polone_{\tixone-1}, \poltwo_{\tixone-1}} }_1 \le \frac{\ratechange_1 +  \influencetwo \cdot \ratechange_2}{1-e^{-\frac{1}{\mixtime}}}.
\end{align*}
Furthermore, for any policy $\polone$: 
\begin{align*}
\norm{\statdist_{\polone, \poltwo_{\tixone}} - \statdist_{\polone, \poltwo_{\tixone-1}} }_1 \le \frac{\influencetwo \cdot \ratechange_2}{1-e^{-\frac{1}{\mixtime}}}.
\end{align*}
\end{lemma}
\iftoggle{longversion}{
\begin{proof}
See \iftoggle{longversion}{Appendix~\ref{appendix.proof_lm_stat_dist_inequality}}{Appendix~E.3  in \cite{radanovic2019learning}}.
\end{proof}
}



\subsection{Properties Based on OFTRL}

The bounds on the change magnitudes of distributions $\statdist_{\polone, \poltwo_{\tixone}}$ and $\statdist_{\polonet, \poltwo_{\tixone}}$, which will propagate to the final result, depend on \agentone's policy change magnitude $\ratechangeone$. The following lemma provides a bound for $\ratechangeone$ that, together with the assumed bound on $\ratechangetwo$, is useful in establishing no-regret guarantees. 
\begin{lemma}\label{lm_weight_diff}
For any $\tixone > 1$ and $1 < \tixtwo \le \Segment$, the change magnitude of weights $\weights_{\tixone, \tixtwo}$ in \algexpsmooth~is bounded by:
\begin{align*} 
\norm{\weights_{\tixone, \tixtwo} - \weights_{\tixone-1, \tixtwo-1}}_{\infty} \le \min \left \{ 2,  \frac{9 \cdot \ratelearn}{ 1 - e^{-\frac{1}{\mixtime}}} \right \}
\end{align*}
Consequently: 
\begin{align*}
\ratechangeone \le \min \left \{2,  \frac{9 \cdot \ratelearn}{ 1 - e^{-\frac{1}{\mixtime}}} + \frac{2}{\Segment} \right \}. 
\end{align*}
\end{lemma}
\iftoggle{longversion}{
\begin{proof} 
See \iftoggle{longversion}{Appendix~\ref{appendix.proof_lm_weight_diff}}{Appendix~E.2 in \cite{radanovic2019learning}}.
\end{proof}
}

Now, we turn to bounding the term $\vecdot{\polone - \polonet, \qvals_{\tixone}}$. 
Lemma~\ref{lm_oftrl_property_simple} formalizes the RVU property for \algexpsmooth~using the $L_1$ norm and its dual $L_{\infty}$ norm, derived from results in the existing literature~\cite{syrgkanis2015fast}.\footnote{An extended version of the lemma, which is  needed for the main result, is provided in \iftoggle{longversion}{Appendix~\ref{appendix.proof_lm_oftrl_property}}{Appendix~E.1 of \citet{radanovic2019learning}}.}
 Lemma \ref{lm_oftrl_property_simple} shows that it is possible to bound $\vecdot{\polone - \polonet, \qvals_{\tixone}}$ by examining the change magnitudes of $\qval$-values.


\begin{lemma}\label{lm_oftrl_property_simple}
Consider \algexpsmooth~and let  $\mathbf 1$ denote column vector of ones with
$|\states|$ elements. Then, 
for each episode $ 1 \le \tixone \le \tix - \Segment + 1$ of \algexpsmooth~, we have: 
\begin{align*}
\sum_{\tixtwo=1}^{\Segment}
& \vecdot{\polone - \weights_{\tixone + \tixtwo-1, \tixtwo},  \qvals_{\tixone + \tixtwo-1}} \\
&\le \mathbf 1 \cdot \left (  \frac{\Delta_{\regul}}{\ratelearn} + \ratelearn \cdot \sum_{\tixtwo=1}^{\Segment} \norm{ \qvals_{\tixone + \tixtwo-1} -  \qvals_{\tixone + \tixtwo - 2}}_{\max}^2 \right . \\
&\left . - \frac{1}{4 \cdot \ratelearn} \cdot \sum_{\tixtwo=1}^{\Segment}  \norm{\weights_{\tixone + \tixtwo-1, \tixtwo}- \weights_{\tixone + \tixtwo-2, \tixtwo-1}}_{\infty}^2  \right ),
\end{align*}
 where $\weights_{\tixone + \tixtwo-1, \tixtwo}$ are defined in \eqref{eq_update_windowing_policy}, $\Delta_{\regul} = \sup_{\weights \in \simplexone} \regul(\weights) -  \inf_{\weights \in \simplexone} \regul(\weights)$, and
 $\polone$ is an arbitrary policy of \agentone.
\end{lemma}
\iftoggle{longversion}{
\begin{proof} 
See \iftoggle{longversion}{Appendix~\ref{appendix.proof_lm_oftrl_property}}{Appendix~E.1 in \cite{radanovic2019learning}}.
\end{proof}
}

\subsection{Change Magnitudes of $\qval$-values}



We  now derive bounds on the change magnitudes of $\qval$-values that we use together with Lemma \ref{lm_oftrl_property_simple} to prove the main results. 
We first bound the difference $\qvals_{\tixone}^{\polonet} - \qvals_{\tixone - 1}^{\polone_{\tixone-1}}$, 
which helps us in bounding the difference $\qvalst -  \qvals_{\tixone-1}$. 
\begin{lemma}\label{lm_qvals_diff_pol}
The difference between $\qvals_{\tixone}^{\polonet}$-values of two consecutive episodes is upper bounded by:
\begin{align*}
\norm{\qvals_{\tixone}^{\polonet} - \qvals_{\tixone - 1}^{\polone_{\tixone-1}}}_{\infty} \le C_{\qval^{\pol}},
\end{align*}
where $C_{\qval^{\pol}} = 3 \cdot \frac{\ratechangeone +  \influencetwo \cdot \ratechangetwo}{(1-e^{-\frac{1}{\mixtime}})^2} + 2 \cdot \frac{\ratechangeone+ \ratechangetwo}{1-e^{-\frac{1}{\mixtime}}}$.
\end{lemma}
\iftoggle{longversion}{
\begin{proof}
See 
\iftoggle{longversion}{Appendix~\ref{appendix.proof_lm_qvals_diff_pol}}{Appendix~E.4 in \cite{radanovic2019learning}}
for the proof.
\end{proof}
}
{
\vspace{2mm}
}


\begin{lemma}\label{lm_qvals_diff_t}
The difference between $\qvals_{\tixone}$-values of two consecutive episodes is upper bounded by:
\begin{align*}
\norm{ \qvalst -  \qvals_{\tixone-1}}_{\max}^2 \le C_{\qval}^2, 
\end{align*}
where $C_{\qval} = C_{\qval^{\pol}} + \left (\frac{3}{1 - e^{\frac{1}{ \mixtime}}} + 1 \right ) \cdot \ratechangeone + 2 \cdot \ratechangetwo + \frac{\ratechange_1 +  \influencetwo \cdot \ratechange_2}{1-e^{-\frac{1}{\mixtime}}}$.
\end{lemma}
\iftoggle{longversion}{
\begin{proof}
See 
\iftoggle{longversion}{Appendix~\ref{appendix.proof_lm_qvals_diff_t}}{Appendix~E.5 in \cite{radanovic2019learning}}
 for the proof.
\end{proof}
}

For convenience, instead of directly using $C_{\qval}$, we consider a variable $C_{\mixtime}$ such that $C_{\mixtime} \ge \frac{C_{\qval}}{\max \{ \ratechangeone, \ratechangetwo \}}$.
The following proposition gives a rather loose (but easy to interpret) bound on $C_{\mixtime}$ that satisfies the  inequality. 

\begin{proposition}\label{prop_C_bound}
There exists a constant $C_{\mixtime}$ independent of  $\ratechangeone$ and $\ratechangetwo$, such that:\footnote{When $\ratechangeone = \ratechangetwo = 0$, $C_{\mixtime} \ge 1$. }
\begin{align*}
\frac{C_{\qval}}{\max \{\ratechangeone, \ratechangetwo \}} \le C_{\mixtime} \le \frac{18}{(1-e^{-\frac{1}{\mixtime}})^2}
\end{align*}
\end{proposition}
\begin{proof}
The claim is directly obtained from Lemma~\ref{lm_qvals_diff_pol}, Lemma~ \ref{lm_qvals_diff_t}, 
and the fact that $\influencetwo \le 1$
and $\frac{1}{1-e^{-\frac{1}{\mixtime}}} \ge 1$.
\end{proof}

\subsection{Regret Analysis and Main Results}\label{sec.subsection4-regret-analysis}
We now come to the most important part of our analysis: establishing the regret guarantees for \algexpsmooth. 
Using the results from the previous subsections, we obtain 
the following regret bound:

\begin{theorem}\label{thm_regret}
%
Let the learning rate of \algexpsmooth~be 
equal to $\ratelearn = \frac{1}{\Segment^{\frac{1}{4}}}$ and let $k > 0$ be such that  
$\ratechangeone \le k \cdot \ratelearn$ and $\ratechangetwo \le k \cdot \ratelearn$. Then, the regret of \algexpsmooth~is upper-bounded by:
\begin{align*}
\regret(\tix) &\le 2 \cdot  \left (\Delta_{\regul}  +  k^2 \cdot C_{\mixtime}^2 \right ) \cdot T \cdot \Segment^{-\frac{3}{4}} \nonumber \\
&+ \frac{ 6 \cdot \influencetwo \cdot \ratechange_2}{(1-e^{-\frac{1}{\mixtime}})^2} \cdot \tix  \cdot \Segment.
\end{align*}
\end{theorem}
\iftoggle{longversion}{
\begin{proof}
See \iftoggle{longversion}{Appendix~\ref{appendix.proof_thm_regret}}{Appendix~F in \cite{radanovic2019learning}} for the proof.
\end{proof}
\vspace{-3mm}
}
{
\vspace{-2mm}
}

When \agenttwo~does not influence the transition kernel through its policy, i.e., when $\influence_2=0$, the regret is  $\bigObound(\tix^{\frac{1}{4}})$ for $\Segment = \tix$. 
In this case, we could have also applied the original approach of \citet{even2005experts,even2009online}, but interestingly, it would result in a worse regret bound, i.e., $\bigObound(\tix^{\frac{1}{2}})$. By leveraging the fact that \agenttwo's policy is slowly changing, 
which corresponds to reward functions in the setting of \citet{even2005experts,even2009online} not being fully adversarial, we are able to improve on the worst-case guarantees. The main reason for such an improvement is our choice of the underlying experts algorithm, i.e., OFTRL, that  
exploits the apparent predictability of \agenttwo's behavior. Similar arguments were made for the repeated games settings \cite{rakhlin2013optimization,syrgkanis2015fast}, which 
correspond to our setting when the \mdp~consists of only one state.
Namely, in the single state scenario, \agenttwo~does not influence transitions, so the resulting regret is $\bigObound(\tix^{\frac{1}{4}})$, matching the results of \citet{syrgkanis2015fast}.


In general, the regret  depends on $\ratechangetwo$. If $\ratechangetwo = \bigObound(\tix^{-\responsivity})$ with $0 < \alpha \le \frac{7}{4}$, then $\Segment = \bigObound(\tix^{\frac{4}{7}\cdot \alpha})$ equalizes the order of the two regret components in Theorem~\ref{thm_regret} and leads to the regret of $\bigObound(\tix^{1-\frac{3}{7}\cdot \alpha})$. This brings us to the main result, which  provides a lower bound on the return $\bar \return$:
\begin{theorem}\label{thm_main}
 Assume that $\ratechangetwo = \bigObound(\tix^{-\responsivity})$ for $\responsivity > 0$. Let $\ratelearn = \frac{1}{\Segment^{\frac{1}{4}}}$ and $\Segment = \min\{\tix^{\frac{4}{7}\cdot \alpha}, \tix\}$. Then, the regret of 
\algexpsmooth~is upper-bounded by:
\begin{align*}
\regret(\tix)  = \bigObound(\tix^{\max\{1-\frac{3}{7}\cdot \alpha, \frac{1}{4}\}}).
\end{align*} 
Furthermore, when the \mdp~is {\em ($\smoothone,\smoothtwo$)-smooth}, the return of 
\algexpsmooth~is lower-bounded by: 
\begin{align*}
\bar \return &\ge \frac{\smoothone}{1+ \smoothtwo} \cdot \opt - \bigObound(\tix^{\max\{-\frac{3}{7}\cdot \alpha, -\frac{3}{4}\}}) - \bigObound(M^{-1}).
\end{align*}
\end{theorem}
\begin{proof}


Notice that $\frac{9 \cdot \ratelearn}{1-e^{-\frac{1}{\mixtime}}} \ge \frac{2}{\Segment}$ for $\Segment \ge 1$. By Lemma \ref{lm_weight_diff}, this implies that there exists a fixed $k$ (not dependant on $\tix$) such that $\ratechangeone \le k \cdot \ratelearn$ for large enough $\tix$. Furthermore, 
$\ratelearn = \tix^{-\min \{\frac{4}{7}\cdot \alpha, \frac{1}{4} \}}$,
so there exists a fixed $k$ such that $\ratechangetwo \le k \cdot \ratelearn$ for large enough $\tix$. Hence, we can apply Theorem \ref{thm_regret} to obtain an order-wise regret bound: $\bigObound(\tix \cdot \Segment^{-\frac{3}{4}}) + \bigObound(\ratechangetwo \cdot \tix \cdot \Segment)$.

Now, consider two cases. First, let $\responsivity \le \frac{7}{4}$. Then, we obtain:
\begin{align*}
&\regret(\tix)=\bigObound(\tix \cdot \Segment^{-\frac{3}{4}}) + \bigObound(\ratechangetwo \cdot \tix \cdot \Segment) \\
&= \bigObound(\tix \cdot \tix^{-\frac{3}{4}\cdot \frac{4}{7}\cdot \responsivity}) + 
\bigObound(\tix^{-\responsivity} \cdot \tix \cdot \tix^{\frac{4}{7}\cdot \responsivity }) = \bigObound(\tix^{1-\frac{3}{7}\cdot \responsivity}).
\end{align*}
For the other case, i.e., when $\responsivity \ge \frac{7}{4}$, we obtain:
\begin{align*}
\regret(\tix)&=\bigObound(\tix \cdot \Segment^{-\frac{3}{4}}) + \bigObound(\ratechangetwo \cdot \tix \cdot \Segment) \\
&= \bigObound(\tix \cdot \tix^{-\frac{3}{4}}) + 
\bigObound(\tix^{-\responsivity} \cdot \tix \cdot \tix)= \bigObound(\tix^{\frac{1}{4}}).
\end{align*}
Therefore, $\regret(\tix) = \bigObound(T^{\max \{ 1-\frac{3}{7}\cdot \responsivity, \frac{1}{4}\}})$, which proves the first statement. By combining it with Lemma \ref{lemma_large_rounds}, we obtain the second statement.
\end{proof}
\vspace{-3mm}
The multiplicative factors in the asymptotic bounds mainly depend on mixing time $\mixtime$. In particular they are dominated by factor $\frac{1}{1-e^{-\frac{1}{\mixtime}}}$ and its powers, as can be seen from Lemma \ref{lemma_large_rounds}, Theorem \ref{thm_regret}, and Proposition \ref{prop_C_bound}. 
Note that Lemma \ref{lm_weight_diff} allows us to upper bound $k$ in Theorem \ref{thm_regret} with $\bigObound \left (\frac{1}{1-e^{-\frac{1}{\mixtime}}} \right )$.
Furthermore, $\frac{1}{1-e^{-\frac{1}{\mixtime}}} \approx \mixtime$ for large enough $\mixtime$.
%
Hence, these results
imply $\bigObound(\mixtime^6)$ dependency of the asymptotic bounds on $\mixtime$. This is larger than what one might expect from the prior work, for example the bound in  \citet{even2005experts,even2009online} has $\bigObound(\mixtime^2)$ dependency. However, our setting is different, in that the presence of \agenttwo~has an effect on transitions (from \agentone's perspective),  and so it is not surprising that the resulting dependency on the mixing time is worse.
\section{Hardness Result}\label{sec.impossibility}
Our formal guarantees assume that the policy change magnitude $\ratechangetwo$ of \agenttwo~is a decreasing function in the number of episodes given by $\bigObound(\tix^{-\responsivity})$ for $\responsivity > 0$. What if we relax this, and allow \agenttwo~to adapt independently of the number of episodes? 
We show a hardness result for
the  setting of $\responsivity = 0$, using a reduction from 
the {\em online agnostic parity learning} problem~\cite{abbasi2013online}.
As argued in \citet{abbasi2013online}, the online to batch reduction implies that the online version of agnostic parity learning is at least as hard as its offline version, for which the best known algorithm has complexity $2^{\bigObound\left (\frac{n}{\log n} \right)}$ \cite{kalai2008agnostic}. In fact, agnostic parity learning is a harder variant of the {\em learning with parity noise} problem, widely believed to be computationally intractable ~\cite{blum2003noise,pietrzak2012cryptography}, and thus often adopted as a hardness assumption (e.g., \citet{sharan2018prediction}). 





\begin{theorem}\label{thm_hardness}
Assume that the policy change magnitude $\ratechangetwo$ of \agenttwo~is order $\Omegabound(1)$ and that its influence is $\influencetwo =1$. 
If there exists a $poly(|\states|, \tix)$ time algorithm that outputs a policy sequence $\polone_1$, ..., $\polone_\tix$ whose regret is 
$\bigObound(poly(|\states|) \cdot \tix^{\beta})$ for $\beta < 1$, then 
there also exists a $poly(|\states|, \tix)$ time algorithm for online agnostic parity learning whose regret is $\bigObound(poly(|\states|) \cdot \tix^{\beta})$. 
\end{theorem}
\vspace{-3mm}
\iftoggle{longversion}{
\begin{proof}
See \iftoggle{longversion}{Appendix~\ref{appendix.proof_thm_hardness}}{Appendix~H in \cite{radanovic2019learning}}.
\end{proof}
\vspace{-3mm}
%
}
{
\vspace{1mm}
\looseness-1
}

The proof relies on the result of \citet{abbasi2013online} (Theorem 5), which reduces the {\em online agnostic parity learning} problem to the {\em adversarial shortest path} problem, which we reduce to our problem. 
Theorem~\ref{thm_hardness} implies that when $\responsivity = 0$, it is unlikely to obtain $\regret(\tix)$ that is sub-linear in $\tix$ given the current computational complexity results. 

\section{Related Work}\label{sec.relatedwork}

\textbf{Experts Learning in MDPs. }
Our framework is closely related to that of~\citet{even2005experts,even2009online}, although the presence of \agenttwo~means that we cannot directly use their algorithmic approach.
In fact, learning with an arbitrarily changing transition is 
believed to be computationally intractable \cite{abbasi2013online}, 
and  computationally efficient learning algorithms experience linear regret~\cite{yu2009arbitrarily,abbasi2013online}. 
%
%
This is where we make use of the bound on the magnitude of \agenttwo's policy change. 
Contrary to most of the existing work,  the changes in reward and transition kernel in our model are non-oblivious and adapting to the learning algorithm of 
\agentone. 
There have been a number of follow-up works that either extend these results or improve them for more specialized settings 
\cite{dick2014online,neu2012adversarial,neu2010online,dekel2013better}. \citet{agarwal2017corralling} and~\citet{singla2018learning} study the problem of learning with experts advice where experts are not stationary and are learning agents themselves. However, their focus is on designing a meta-algorithm on how to coordinate with these experts and is technically very different from ours.

\vspace{-1mm}
\textbf{Learning in Games.}
To relate the quality of  an optimal solution to \agentone's regret, we use techniques similar to those studied in the learning in games literature \cite{blum2008regret,roughgarden2009intrinsic,syrgkanis2015fast}.  The fact that \agenttwo's policy is changing slowly  enables us to utilize no-regret algorithms for learning in games with recency bias~\cite{daskalakis2011near,rakhlin2013optimization,syrgkanis2015fast}, providing better regret bounds than through standard no-regret learning techniques \cite{littlestone1994weighted,freund1997decision}. 
The recent work by \citet{wei2017online} studies two-player learning in zero-sum stochastic games. Apart from focusing on zero-sum games,  \citet{wei2017online} adopt a different set of assumptions to derive  regret bounds and their results are not directly comparable to ours. Furthermore, their algorithmic techniques are orthogonal to those that we pursue; these differences are elaborated in \citet{wei2017online}. 
\vspace{-1mm}
\textbf{Human AI Collaboration.}
The helper-AI problem~\cite{dimitrakakis2017multi} is related to the present work, 
in that an AI agent is designing its policy by accounting for human imperfections. The authors use a Stackleberg formulation of the  problem in a single shot scenario. Their model assumes that the AI agent knows the behavioral model of the human agent, which 
is a best response to the policy of the AI agent for an incorrect transition kernel. We relax this requirement by studying a repeated human-AI interaction.  \citet{nikolaidis2017game} study a repeated human-AI interaction, but their setting is more restrictive than ours as they do not model the changes in the environment. In particular, they have a repeated game setup, where the only aspect that changes over time is the ``state" of the human representing what knowledge the human has about the robot's payoffs. 
Prior work also considers a learner that is aware of the presence of other actors~\cite{foerster2018learning,raileanu2018modeling}. While these multi-agent learning approaches account for the evolving behavior of other actors, the underlying assumption is typically that each agent follows a known model. 


\vspace{-1mm}
\textbf{Steering and Teaching.}
There is also a related literature on ``steering'' the behavior of other agent.
For example, (i) the \textit{environment design} framework of~\citet{zhang2009policy}, where  one agent tries to steer the behavior of another agent by modifying its reward function, (ii) the \textit{cooperative inverse reinforcement learning} of~\citet{hadfield2016cooperative}, where the human uses demonstrations to reveal a proper reward function to the AI agent, and (iii) the \textit{advice-based interaction} model~\cite{amir2016interactive},  where the goal is to communicate advice to a sub-optimal agent on how to act in the world. The latter approach is also
in close relationship to the {\em machine teaching} literature~\cite{zhu2018overview,zhu2015machine,singla2013actively,cakmak2012algorithmic}.
Our work differs from this literature; we focus on joint decision-making, rather than teaching or steering.
\section{Conclusion}\label{sec.conclusions}
\vspace{-1mm}
In this paper, we  have presented a two-agent MDP framework in a
collaborative setting. We considered the problem of designing a
no-regret algorithm for the first agent in the presence of an
adapting, second agent (for which we make no assumptions
about its behavior other than a requirement that it adapts slowly enough). Our algorithm builds from the ideas of experts learning in MDPs, and makes use of a novel form of recency bias to achieve strong regret bounds. In particular, we showed that in order for 
the first agent to facilitate  collaboration, it is critical that 
the second agent's policy changes are not abrupt.
An interesting direction for future work would be to consider the partial information setting, in which, for example, \agentone~has only a
noisy estimate of \agenttwo's policy.  



\section*{Acknowledgements}

This work was supported in part by a SNSF Early Postdoc Mobility fellowship.

\bibliographystyle{icml2019}
\bibliography{references}

\iftoggle{longversion}{
\clearpage
\onecolumn
\appendix 
{\allowdisplaybreaks


\clearpage

\section{List of Appendices}

In this section we provide a brief description of the the content provided in the appendices of the paper.   
\begin{enumerate}
\item Appendix \ref{appendix.useful_properties} contains the statements and the corresponding proofs for the MDP properties that we used in proving the technical results of the paper.  
\item Appendix \ref{appendix.connection_to_smoothenss} provides a relationship between the smoothness parameters introduced in Section \ref{sec.approach} and structural properties of our setting.  
\item Appendix \ref{appendix.proof_lemma_large_rounds} provides the proof of Lemma \ref{lemma_large_rounds}, which connects no-regret learning with the optimization objective (see Section \ref{sec.approach}).
\item Appendix \ref{appendix.lemmas_for_regret} provides the proofs of the lemmas related to our algorithmic approach that are important for proving the main results. 
\item Appendix \ref{appendix.proof_thm_regret} provides the proof of Theorem \ref{thm_regret}, which establishes the regret bound of Algorithm \ref{algo_leader_oftrl} (see Section \ref{sec.subsection4-regret-analysis}). 
\item Appendix \ref{appendix.properties_algo_oftrl_rand_restart} describes the properties of experts with periodic restarts (see Section \ref{sec.experts_restarts}).
\item Appendix \ref{appendix.proof_thm_hardness} provides the proof of Theorem \ref{thm_hardness}, which establishes the hardness of achieving no-regret if \agenttwo's policy change is not a decreasing function in the number of episodes (see Section \ref{sec.impossibility}).    
\end{enumerate}

\clearpage

\section{Important \mdp~properties}\label{appendix.useful_properties}

To obtain our formal results, we derive several useful \mdp~properties.

\subsection{Policy-reward bounds}

The first policy reward bound we show is the upper-bound on the vector product between $\polonet$ and $\rewst$.

\begin{lemma}\label{lm_max_pol_rew}
The policy-reward dot product is bounded by:
\begin{align*}
\norm{\vecdot{\polonet, \rewst}}_{\infty} \le 1.
\end{align*}
\end{lemma}
\begin{proof}
The definition of $\vecdot{,}$ gives us:
\begin{align*}
\norm{\vecdot{\polonet, \rewst}}_{\infty} &= \max_{\state} |\sum_{\actionone \in \actions}\polonet(\state, \action) \cdot \rewst(\state, \action)| 
\le \max_{\state} \sum_{\actionone \in \actions} |\polonet(\state, \action) \cdot \rewst(\state, \action)| 
 \le\max_{\state} \sum_{\actionone \in \actions} \polonet(\state, \action) \cdot |\rewst(\state, \action)| \\
& \le \max_{\state} \sum_{\actionone \in \actions} \polonet(\state, \action)  =  1 ,
\end{align*}
where we used the triangle inequality and the boundedness of rewards. 
\end{proof}

Note that the lemma holds for any $\polonet$ and $\rewst$ (i.e., for any $\polone$ and $\rews$). 
Furthermore, using the following two lemmas, we also bound the difference between two consecutive policy-reward dot products. In particular:

\begin{lemma}\label{lm_diff_pol_rew}
The following holds:
\begin{align*}
\norm{\rewst - \rews_{\tixone-1}}_{\max} \le \norm{\poltwot - \poltwo_{\tixone-1}}_\infty &\le \ratechangetwo,\\
\norm{\vecdot{\polonet, \rewst} -  \vecdot{ \polone_{\tixone},  \rews_{\tixone-1}}}_{\infty} &\le \ratechangetwo,\\
\norm{\vecdot{\polonet, \rews_{\tixone-1}} -  \vecdot{ \polone_{\tixone-1},  \rews_{\tixone-1}}}_{\infty} &\le \ratechangeone,
\end{align*}
for all $\tixone > 1$.
\end{lemma}
\begin{proof}
To obtain the first inequality, note that: 
\begin{align*}
&\norm{\rewst - \rews_{\tixone-1}}_{\max}
= \max_{\state,\actionone} |\rewst(\state, \actionone) - \rews_{\tixone-1}(\state, \actionone)|
 \le \max_{\state,\actionone} |\sum_{\actiontwo \in \actionstwo} \rew(\state, \actionone,\actiontwo) \cdot (\poltwot(\state,\actiontwo)- \poltwo_{\tixone-1}(\state,\actiontwo))| \\
& \le \max_{\state} \sum_{\actiontwo \in \actionstwo} |\poltwot(\state,\actiontwo)- \poltwo_{\tixone-1}(\state,\actiontwo)|   
 = \norm{\poltwot -\poltwo_{\tixone-1}}_{\infty}  \le \ratechangetwo,
\end{align*}
where we used the triangle inequality and the fact that $|\rew(\state,\actionone, \actiontwo)| \le 1$. The second inequality holds because:
\begin{align*}
&\norm{\vecdot{\polonet, \rewst} -  \vecdot{ \polone_{\tixone},  \rews_{\tixone-1}}}_{\infty} =
\norm{\vecdot{\polonet, \rewst - \rews_{\tixone-1}}}_{\infty}  
 = \max_{\state} |\sum_{\actionone \in \actionsone} \polonet(\state,\actionone) \cdot (\rewst(\state, \actionone) - \rews_{\tixone-1}(\state, \actionone))|\\
&\le \max_{\state, \actionone} |\rewst(\state, \actionone) - \rews_{\tixone-1}(\state, \actionone)| = \norm{\rewst - \rews_{\tixone-1}}_{\max} \le \ratechangetwo,
\end{align*}
where we used the triangle inequality, the fact that $\polonet(\state,\actionone) \in \simplexone$, and the first inequality of the lemma (proven above).
Finally, we have that:
\begin{align*}
&\norm{\vecdot{\polonet, \rews_{\tixone-1}} -  \vecdot{ \polone_{\tixone-1},  \rews_{\tixone-1}}}_{\infty}
=\norm{\vecdot{\polonet - \polone_{\tixone-1}, \rews_{\tixone-1}}}_{\infty}  
= \max_{\state} |\sum_{\actionone \in \actionsone} (\polonet(\state,\actionone) - \polone_{\tixone-1}(\state,\actionone)) \cdot  \rews_{\tixone-1}(\state, \actionone)|\\
&\le \max_{\state} \sum_{\actionone \in \actionsone} |\polonet(\state,\actionone) - \polone_{\tixone-1}(\state,\actionone)| 
= \norm{\poltwot - \poltwo_{\tixone-1}}_\infty \le \ratechangeone,
\end{align*}
where again we used the triangle inequality and the fact that $|\rewst(\state,\actionone)| \le 1$ (boundedness of rewards).
\end{proof}

We will also need a bit different, albeit similar, statement when it comes to the third property of Lemma \ref{lm_diff_pol_rew}.
\begin{lemma}\label{lm_diff_pol_rew_general}
Let $\bm p$, $\bm p'$, and $\bm M$ be matrix of dimension $|\states| \times |\actionsone|$, with restriction that the rows of 
$\bm p$ and $\bm p'$ are elements of probability simplex $\simplexone$, i.e., $\bm p(s), \bm p'(s) \in \simplexone$. 
The following holds:
\begin{align*}
\norm{\vecdot{\bm p, \bm M} -  \vecdot{\bm p', \bm M}}_{\infty} &\le 2 \cdot \norm{\bm M}_{\max}.
\end{align*}
\end{lemma}
\begin{proof}
Similar to the third claim in the previous lemma:
\begin{align*}
&\norm{\vecdot{\bm p,  \bm M} -  \vecdot{\bm p',  \bm M}}_{\infty} = \norm{\vecdot{\bm p - \bm p',  \bm M}}_{\infty} 
\le \max_{\state} |\sum_{\actionone \in \actionsone} (\bm p(\state,\actionone) - \bm p'(\state,\actionone)) \cdot  \bm M(\state, \actionone)|\\
&\le \max_{\state} \sum_{\actionone \in \actionsone} |\bm p(\state,\actionone) - \bm p'(\state,\actionone)| \cdot  \norm{\bm M}_{\max}
= 2 \cdot  \norm{\bm M}_{\max},
\end{align*}
where the last inequality is obtained from the fact that $\bm p(s), \bm p'(s) \in \simplexone$.
\end{proof}

\subsection{Transition kernel bounds}
The following lemma provides bounds on the effective transition kernel change due to the agents' policy changes. 

\begin{lemma}\label{lm_trans_bound}
For generic policies $\polone$, $\polone'$, $\poltwo$, and $\poltwo'$, we have:
\begin{align*}
\norm{\trans_{\polone, \poltwo} - \trans_{\polone, \poltwo'}}_{\infty} &\le \norm{\poltwo - \poltwo'}_{\infty} \le \ratechangetwo,\\
\norm{\trans_{\polone, \poltwo} - \trans_{\polone', \poltwo}}_{\infty} &\le \norm{\polone - \polone'}_{\infty} \le \ratechangeone
\end{align*}
\end{lemma}
\begin{proof}
W.l.o.g., we restrict our analysis to the first inequality. 
We have:
\begin{align*}
&\norm{\trans_{\polone, \poltwo} - \trans_{\polone, \poltwo'}}_{\infty} 
= \max_{\state} \sum_{\state_{new} \in \states} 
|\sum_{\actionone \in \actionsone, \actiontwo \in \actionstwo}\trans(\state, \actionone, \actiontwo, \state_{new}) \cdot \polone(\state,\actionone) \cdot (\poltwo(\state, \actiontwo) - \poltwo'(\state, \actiontwo)) | \\  
 &\le \max_{\state} \sum_{\state_{new} \in \states, \actionone \in \actionsone, \actiontwo \in \actionstwo}
 \trans(\state, \actionone, \actiontwo, \state_{new}) \cdot \polone(\state,\actionone) \cdot |\poltwo(\state, \actiontwo) - \poltwo'(\state, \actiontwo)| \\ 
 &= \max_{\state} \sum_{\actionone \in \actionsone, \actiontwo \in \actionstwo} \polone(\state,\actionone) \cdot |\poltwo(\state, \actiontwo) - \poltwo'(\state, \actiontwo)| 
 =  \max_{\state} \sum_{\actiontwo \in \actionstwo} |\poltwo(\state, \actiontwo) - \poltwo'(\state, \actiontwo)| \\
 &= \norm{\poltwo - \poltwo'}_{\infty} \le \ratechangetwo. 
\end{align*}
where we used the fact that $\polone \in \simplexone$, while $\trans(\state, \actionone, \actiontwo, \state_{new})$ is an element of the probability 
simplex over the state space for given $\state$, $\actionone$ and $\actiontwo$.
\end{proof}

The direct consequence of the lemma is the bound on \agenttwo's influence:

\begin{corollary}\label{cor_influence_bound}
Influence $\influencetwo$ of \agenttwo~ takes values in $[0, 1]$.
\end{corollary}
\begin{proof}
The lower bound of $\influencetwo$ follows trivially from the definition. The upper bound is obtain from Lemma \ref{lm_trans_bound}:
\begin{align*}
\influencetwo =  \max_{\polone, \poltwo \ne \poltwo'} \frac{\norm{\trans_{\polone,\poltwo} - \trans_{\polone,\poltwo'}}_{\infty}}{\norm{\poltwo - \poltwo'}_{\infty}} 
\le \max_{\polone, \poltwo \ne \poltwo'} \frac{\norm{\poltwo - \poltwo'}_{\infty}}{\norm{\poltwo - \poltwo'}_{\infty}} = 1.
\end{align*}
\end{proof}

An analogous result holds for \agentone's influence. Moreover, we provide bounds for kernel $\trans_{\poltwo_{\tixone}}(\state, \actionone)$ and its change magnitude: 

\begin{lemma}\label{lm_trans_bound_2}
For generic policies $\poltwo$, and $\poltwo'$, we have:
\begin{align*}
\norm{\trans_{\poltwo}(\state, \actionone) - \trans_{\poltwo'}(\state, \actionone)}_{1} &\le \norm{\poltwo(\state) - \poltwo'(\state)}_{1},
\end{align*}
\end{lemma}
\begin{proof}
By the definition of $\trans(\state,\actionone,\actiontwo)$, we know that $\norm{\trans(\state,\actionone,\actiontwo)}_{1} = 1$. Therefore:
\begin{align*}
&\norm{\trans_{\poltwo}(\state, \actionone) - \trans_{\poltwo'}(\state, \actionone)}_{1} 
\le \norm{\sum_{\actiontwo\in\actionstwo}(\poltwo(\state,\actiontwo) - \poltwo'(\state,\actiontwo)) \cdot \trans(\state,\actionone,\actiontwo)}_1\\
&\rightarrow \textit{By the triangle inequality } \\
&\le \sum_{\actiontwo\in\actionstwo} |\poltwo(\state,\actiontwo) - \poltwo'(\state,\actiontwo)| \cdot \norm{\trans(\state,\actionone,\actiontwo)}_1
\le \sum_{\actiontwo\in\actionstwo} |\poltwo(\state,\actiontwo) - \poltwo'(\state,\actiontwo)| = \norm{\poltwo(\state) - \poltwo'(\state)}_{1}.
\end{align*}
\end{proof}

\subsection{State distribution bounds}

In this subsection, we develop a series of bounds on state distribution difference important for the development of our formal results. 
The first result quantifies the change in the state distribution difference due to the change of \agenttwo~ policy.

\begin{lemma}\label{lm_dist_bound_1}
For generic policies $\polone$, $\polone'$, $\poltwo$ and $\poltwo'$, we have:
\begin{align*}
\norm{\statdist_{\polone, \poltwo} - \statdist_{\polone, \poltwo'} }_1 \le \frac{1}{1-e^{-\frac{1}{\mixtime}}} \cdot \norm{\trans_{\polone, \poltwo} - \trans_{\polone, \poltwo'}}_{\infty},\\
\norm{\statdist_{\polone, \poltwo} - \statdist_{\polone', \poltwo} }_1 \le \frac{1}{1-e^{-\frac{1}{\mixtime}}} \cdot \norm{\trans_{\polone, \poltwo} - \trans_{\polone', \poltwo}}_{\infty}.
\end{align*}
\end{lemma}
\begin{proof}
W.l.o.g., we restrict our analysis to the first inequality. 
We have:
\begin{align*}
&\norm{\statdist_{\polone, \poltwo} - \statdist_{\polone, \poltwo'} }_1 
=  \norm{\statdist_{\polone, \poltwo} \cdot \trans_{\polone, \poltwo} - \statdist_{\polone, \poltwo'} \cdot \trans_{\polone, \poltwo'}  }_1\\ 
&\rightarrow \textit{ by rearranging } \\
&=  \normL{\statdist_{\polone, \poltwo} \cdot (\trans_{\polone, \poltwo} - \trans_{\polone, \poltwo'})}
+ \normR{(\statdist_{\polone, \poltwo} - \statdist_{\polone, \poltwo'}) \cdot \trans_{\polone, \poltwo'} }_1\\
&\rightarrow \textit{ by triangle inequality } \\
&\le  \norm{\statdist_{\polone, \poltwo} \cdot (\trans_{\polone, \poltwo} - \trans_{\polone, \poltwo'})}_1+ 
\norm{(\statdist_{\polone, \poltwo} - \statdist_{\polone, \poltwo'}) \cdot \trans_{\polone, \poltwo'} }_1\\
&\rightarrow \textit{ Def. operator norm and the mixing assumption } \\
&\le  \norm{\statdist_{\polone, \poltwo}}_1 \cdot \norm{\trans_{\polone, \poltwo} - \trans_{\polone, \poltwo'}}_{\infty}
+ \norm{\statdist_{\polone, \poltwo} - \statdist_{\polone, \poltwo'}}_1 \cdot e^{-\frac{1}{\mixtime}}\\
&\le  \norm{\trans_{\polone, \poltwo} - \trans_{\polone, \poltwo'}}_{\infty} + \norm{\statdist_{\polone, \poltwo} - \statdist_{\polone, \poltwo'}}_1 \cdot e^{-\frac{1}{\mixtime}}.
\end{align*}
By rearranging the terms we obtain the claim.
\end{proof}

Now, let us also bound $L_1$ distance between state distribution $\statdist_{\tixone, \round}$ and stationary distributions $\statdistt$. 
We obtain the following result following the calculations in Lemma 2 of \cite{even2005experts}: 

\begin{lemma}\label{lm_dist_bound_2}
The $L_1$ distance between $\statdist_{\tixone, \round}$ and $\statdistt$ is for any $\tix \ge 1$ bounded by:
\begin{align*}
\norm{\statdist_{\tixone, \round} - \statdist_{\polone_{\tixone}, \poltwo_{\tixone}} }_1 \le 2 \cdot e^{-\frac{\round-1}{\omega}}.
\end{align*}
\end{lemma}
\begin{proof}
Using the properties of $\statdist$ and $\trans$ and the mixing assumption, we have
\begin{align*}
&\norm{\statdist_{\tixone, \round} - \statdist_{\polone_{\tixone}, \poltwo_{\tixone}}}_{1} = \norm{(\statdist_{\tixone, \round-1} - \statdist_{\polone_{\tixone}, \poltwo_{\tixone}}) \cdot \transt}_{1}
 \le \norm{\statdist_{\tixone, \round-1} - \statdist_{\polone_{\tixone}, \poltwo_{\tixone}}}_{1} \cdot e^{-\frac{1}{\omega}} \\
&\rightarrow \textit{ By induction }\\
&\le \norm{\statdist_{\tixone, 1} - \statdist_{\polone_{\tixone}, \poltwo_{\tixone}}}_{1} \cdot e^{-\frac{\round-1}{\omega}} \le 2 \cdot e^{-\frac{\round-1}{\omega}}
\end{align*}
\end{proof}

%

Furthermore, the following lemma describes the relation between the change rate of the distance between state distribution $\statdist_{\tixone, \round}$ and stationary distributions $\statdistt$.
In particular, the lemma provides a bound on the cumulative distributional change magnitude w.r.t. the agents' policy change magnitudes.

\begin{lemma}\label{lm_dist_bound_4}
For $\tixone > 1$, let $S_{\tixone, \rounds} $ denote the cumulative change magnitude of the $L_1$ distance between $\statdist_{\tixone, \round}$ and $\statdistt$ over $\rounds$ rounds:
\begin{align*}
S_{\tixone, \rounds} = \sum_{\round = 1}^{\rounds} \normL{(\statdist_{\tixone, \round} - \statdist_{\polone_{\tixone}, \poltwo_{\tixone}} )} -\normR{ ( \statdist_{\tixone-1, \round}  - \statdist_{\polone_{\tixone-1}, \poltwo_{\tixone-1}} )}_{1}
\end{align*}
Then $\lim_{\rounds \rightarrow \infty} S_{\tixone, \rounds}$ is bounded by: 
\begin{align*}
\lim_{\rounds \rightarrow \infty} S_{\tixone, \rounds} \le 3 \cdot \frac{\ratechangeone +  \influencetwo \cdot \ratechangetwo}{(1-e^{-\frac{1}{\mixtime}})^2}
\end{align*}
\end{lemma}
\begin{proof}
Due to the triangle inequality and Lemma \ref{lm_dist_bound_2}, we have that:
\begin{align*}
\norm{(\statdist_{\tixone, \round} - \statdist_{\polone_{\tixone}, \poltwo_{\tixone}} ) - ( \statdist_{\tixone-1, \round}  - \statdist_{\polone_{\tixone-1}, \poltwo_{\tixone-1}} )}_{1}
 &\le \norm{\statdist_{\tixone, \round} - \statdist_{\polone_{\tixone}, \poltwo_{\tixone}}}_{1} + \norm{\statdist_{\tixone-1, \round}  - \statdist_{\polone_{\tixone-1}, \poltwo_{\tixone-1}}}_{1}\\
&\le 4 \cdot e^{-\frac{\round-1}{\omega}}.
\end{align*}
Therefore, we know $S_{\tixone, \rounds}$ converges absolutely and there exists limit $\lim_{\rounds \rightarrow \infty} S_{\tixone, \rounds}$. 
Furthermore, since $\statdist_{\tixone, 1} = \statdist_{\tixone - 1, 1}$, we have:
\begin{align*}
&S_{\tixone, \rounds} - \norm{ \statdist_{\polone_{\tixone}, \poltwo_{\tixone}} -  \statdist_{\polone_{\tixone-1}, \poltwo_{\tixone-1}}}_{1} 
=  \sum_{\round = 2}^{\rounds} \norm{(\statdist_{\tixone, \round} - \statdist_{\polone_{\tixone}, \poltwo_{\tixone}} ) - ( \statdist_{\tixone-1, \round}  - \statdist_{\polone_{\tixone-1}, \poltwo_{\tixone-1}} )}_{1}  \\ 
&\rightarrow \textit{ By $\statdist \trans = \statdist$ properties }\\
&\le  \sum_{\round = 2}^{\rounds} \norm{(\statdist_{\tixone, \round-1} - \statdist_{\polone_{\tixone}, \poltwo_{\tixone}} ) \cdot \transt - ( \statdist_{\tixone-1, \round-1}  - \statdist_{\polone_{\tixone-1}, \poltwo_{\tixone-1}}) \cdot \trans_{\polone_{\tixone-1}, \poltwo_{\tixone-1}}}_{1}\\
&\rightarrow \textit{ $+$ and $-$ additional terms and the triangle inequality }\\
&\le \sum_{\round = 2}^{\rounds} \norm{((\statdist_{\tixone, \round-1} - \statdist_{\polone_{\tixone}, \poltwo_{\tixone}} ) - ( \statdist_{\tixone-1, \round-1}  - \statdist_{\polone_{\tixone-1}, \poltwo_{\tixone-1}} )) \cdot \transt}_{1} \\
&+ \sum_{\round = 2}^{\rounds} \norm{(\statdist_{\tixone-1, \round-1}  - \statdist_{\polone_{\tixone-1}, \poltwo_{\tixone-1}})  \cdot (\transt - \trans_{\polone_{\tixone-1}, \poltwo_{\tixone-1}} )}_{1} \\
&\rightarrow \textit{Denote:  $d = \frac{1}{2} \cdot (\statdist_{\tixone, \round-1} + \statdist_{\polone_{\tixone-1}, \poltwo_{\tixone-1}})$ and  $d' = \frac{1}{2} \cdot (\statdist_{\polone_{\tixone}, \poltwo_{\tixone}} + \statdist_{\tixone-1, \round-1} )$ }\\
&= \sum_{\round = 2}^{\rounds} 2 \cdot \norm{ (d - d') \cdot \transt}_{1} + \sum_{\round = 2}^{\rounds} \norm{(\statdist_{\tixone-1, \round-1}  - \statdist_{\polone_{\tixone-1}, \poltwo_{\tixone-1}}) \cdot (\transt - \trans_{\polone_{\tixone-1}, \poltwo_{\tixone-1}} )}_{1}\\
&\rightarrow \textit{ Bt the mixing assumption + Holder's inequality and the operator norm definition }\\
&\le \sum_{\round = 2}^{\rounds} 2 \cdot \norm{d - d'}_{1} \cdot e^{-\frac{1}{\mixtime}} + \sum_{\round = 2}^{\rounds} \norm{\statdist_{\tixone-1, \round-1}  - \statdist_{\polone_{\tixone-1}, \poltwo_{\tixone-1}}}_{1} \cdot \norm{\transt - \trans_{\polone_{\tixone-1}, \poltwo_{\tixone-1}}}_{\infty} \\
&\rightarrow \textit{ By the triangle inequality }\\
&\le \sum_{\round = 2}^{\rounds} 2 \cdot \norm{d - d'}_{1} \cdot e^{-\frac{1}{\mixtime}} + \sum_{\round = 2}^{\rounds} \norm{\statdist_{\tixone-1, \round-1}  - \statdist_{\polone_{\tixone-1}, \poltwo_{\tixone-1}}}_{1} \cdot  \left ( \norm{\transt - \trans_{\polone_{\tixone-1}, \poltwo_{\tixone}}}_{\infty} + \norm{\trans_{\polone_{\tixone-1}, \poltwo_{\tixone}} - \trans_{\polone_{\tixone-1}, \poltwo_{\tixone-1}}}_{\infty} \right ) \\
&\rightarrow \textit{ Defs. of $d$ and $d'$, def. of $\influencetwo$, and Lemma \ref{lm_trans_bound}}\\
&\le \sum_{\round = 2}^{\rounds} \norm{(\statdist_{\tixone, \round-1} - \statdist_{\polone_{\tixone}, \poltwo_{\tixone}} )  - ( \statdist_{\tixone-1, \round-1}  - \statdist_{\polone_{\tixone-1}, \poltwo_{\tixone-1}})}_{1} \cdot e^{-\frac{1}{\mixtime}} + \sum_{\round = 2}^{\rounds} \norm{\statdist_{\tixone-1, \round-1}  - \statdist_{\polone_{\tixone-1}, \poltwo_{\tixone-1}}}_{1} \cdot (\ratechangeone + \influencetwo \cdot \ratechangetwo) \\
&\rightarrow \textit{ Relabeling and Lemma \ref{lm_dist_bound_2} }\\
&=  e^{-\frac{1}{\mixtime}} \cdot S_{\rounds - 1} + 2\cdot (\ratechangeone + \influencetwo \cdot \ratechangetwo) \cdot  \sum_{\round = 1}^{\rounds-1} e^{-\frac{\round-1}{\mixtime}}.
\end{align*}
Together with Lemma \ref{lm_stat_dist_inequality}, this gives us:
\begin{align*}
S_{\rounds} - e^{-\frac{1}{\mixtime}} \cdot S_{\rounds - 1} \le \frac{\ratechangeone +  \influencetwo \cdot \ratechangetwo}{1-e^{-\frac{1}{\mixtime}}} + 2\cdot (\ratechangeone + \influencetwo \cdot \ratechangetwo) \cdot  \sum_{\round = 1}^{\rounds-1} e^{-\frac{\round-1}{\mixtime}} \le 3 \cdot \frac{\ratechangeone +  \influencetwo \cdot \ratechangetwo}{1-e^{-\frac{1}{\mixtime}}}.
\end{align*}
By taking the limit $\rounds \rightarrow \infty$, we obtain:
\begin{align*}
\lim_{\rounds \rightarrow \infty} S_{\rounds} \le  3 \cdot \frac{\ratechangeone +  \influencetwo \cdot \ratechangetwo}{(1-e^{-\frac{1}{\mixtime}})^2},
\end{align*}
which completes the proof.
\end{proof}  

\subsection{Average Reward}\label{subsec.avg_reward}

Now, we relate the average reward to $\qval$-values, which is 
important for the analysis of our algorithmic approaches. The following result is an adaptation of Lemma 7 of \cite{even2005experts}  to our setting:  
\begin{lemma}\label{lm_eta_q_relation}
For any joint policy $\bm \pi$, we have:
\begin{align*}
 \avgrewt(\polone)-\avgrewt(\polonet) = \statdist_{\polone, \poltwo_{\tixone}} \cdot ( \qvalst^\polone  - \qvalst^\polonet ) .
\end{align*}
\end{lemma}
\begin{proof}
By the definition of $\qvalst^\polone$:
\begin{align*}
\statdist_{\polone, \poltwo_{\tixone}} \cdot  \qvalst^\polone &=  \statdist_{\polone, \poltwo_{\tixone}} \cdot \vecdot{\polone, \qvalst}.
 \end{align*}
 The Bellman equation gives us:
 \begin{align*}
 \qvalst(\state_1, \actionone) = \rewst(\state_1, \actionone) - \avgrewt(\polonet) + \trans_{\poltwo_{\tixone}}(\state_1, \actionone) \cdot \qvalst^\polonet .
 \end{align*}
Plugging its right hand side into the right hand side of the above equation, we obtain:
 \begin{align*}
\statdist_{\polone, \poltwo_{\tixone}} \cdot  \qvalst^\polone &= \statdist_{\polone, \poltwo_{\tixone}} \cdot \vecdot{\polone,  \rewst} - \statdist_{\polone, \poltwo_{\tixone}} \cdot \mathbf 1 \cdot \avgrewt(\polonet) +  \statdist_{\polone, \poltwo_{\tixone}}  \cdot \trans_{\polone, \poltwo_{\tixone}} \cdot \qvalst^\polonet\\
&= \avgrewt(\polone)-\avgrewt(\polonet) +  \statdist_{\polone, \poltwo_{\tixone}}  \cdot \qvalst^\polonet,
\end{align*}
where $\mathbf 1$ is a column vector of ones with $|\states|$ elements. Rearranging yields the result.
\end{proof}

\subsection{Bound on $\qval$-values}

To make our analysis sound, we also ought to bound the $Q$-values themselves. We can use an approach similar to Lemma 3 of \cite{even2005experts} to obtain:
\begin{lemma}\label{lm_q_bound}
It holds that $\norm{\qvals_{\tixone}^{\polone_{\tixone}} }_\infty \le \frac{2}{1-e^{-\frac{1}{\mixtime}}}$, and consequently  $|\qvals_{\tixone}(\state, \actionone)| \le  \frac{3}{1-e^{-\frac{1}{\mixtime}}}$
and $\norm{\qvals_{\tixone}}_{\max} \le  \frac{3}{1-e^{-\frac{1}{\mixtime}}}$.
\end{lemma}
\begin{proof}
To evaluate $\qvals_{\tixone}^\polonet(\state)$, consider $\statdist_1(\state') = \mathbbm 1_{\state = \state'}$.
Then, from Lemma \ref{lm_dist_bound_2} we know that:
\begin{align*}
&\left |\E \bracket{\rewst(\state_{\round}, \actionone_{\round}) - \avgrew_{\tixone}(\polone_{\tixone})  | \state_1 = \state, \polonet} \right |  = 
\left |\E \bracket{\rewst(\state_{\round}, \actionone_{\round}) - \avgrew_{\tixone}(\polone_{\tixone})  | \statdist_1, \polonet} \right |\\
&=   \left  |\statdistr \cdot \vecdot{ \polone, \rewst} - \statdist_{\polone,\poltwot} \cdot \vecdot{ \polone, \rewst} \right |
\le \norm{\statdist_{\tixone, \round} - \statdistt}_{1} \le  2 \cdot e^{-\frac{\round-1}{\mixtime}}
\end{align*} 
where we used the fact that $\avgrewt(\polone) = \statdist_{\polone,\poltwot} \cdot \vecdot{ \polone, \rewst}$ . Therefore, we obtain that:
\begin{align*}
|\qvals_{\tixone}^\polonet(\state)|  &= \left  |\E \bracket{ \sum_{\round = 1}^{\infty} \rewst(\state_{\round}, \actionone_{\round}) - \avgrew_{\tixone}(\polone_{\tixone})  | \state_1 = \state, \polonet} \right |
\le \E \bracket{ \sum_{\round = 1}^{\infty} \left | \rewst(\state_{\round}, \actionone_{\round}) - \avgrew_{\tixone}(\polone_{\tixone}) \right |  | \state_1 = \state, \polonet} \\
&\le \sum_{\round = 1}^{\infty} 2 \cdot e^{-\frac{\round-1}{\mixtime}} \le \frac{2}{1-e^{-\frac{1}{\mixtime}}},
\end{align*}
which proves the first statement.

The second inequality (and hence, the third) can be obtained from the Bellman's equation:
\begin{align*}
&|\qvalst(\state, \actionone)| = |\rews_t(\state, \actionone) - \avgrew_{\tixone}(\polone_{\tixone}) + \trans_{\poltwo_{\tixone}}(\state, \actionone) \cdot \qvals_{\tixone}^\polonet|\\
&\rightarrow \textit{ By the triangle inequality }\\
&\le | \rews_t(\state, \actionone)-\avgrew_{\tixone}(\polone_{\tixone}) | + |\trans_{\poltwo_{\tixone}}(\state, \actionone) \cdot \qvals_{\tixone}^\polonet| \\
&\rightarrow \textit{ By Holder's inequality}\\
&\le 1 + \norm{\trans_{\poltwo_{\tixone}}(\state, \actionone)}_{1} \cdot \norm{\qvals_{\tixone}^\polonet}_{\infty} \le 1 + \frac{2}{1-e^{-\frac{1}{\mixtime}}} \le \frac{3}{1-e^{-\frac{1}{\mixtime}}}.
\end{align*}
%
\end{proof}

\clearpage

\section{Connection to the smoothness parameters}\label{appendix.connection_to_smoothenss}

Interestingly, we can link the notion of influence to that of the smoothness criterion. To do so, it is useful to define the diameter of the policy space for each of the two agents.
For \agentone (and analogously for \agenttwo), the diameter is defined as: 
\begin{align*}
\diampolone &= \sup_{\polone, \polone' \ne \polone} \sum_{\actionone \in \actions_i} |\polone(\state, \actionone) - \polone'(\state, \actionone)| \\
&= \sup_{\polone, \polone' \ne \polone} \norm{\polone - \polone'}_{\infty}.
\end{align*}

For simplicity of exposure, we will assume that the reward function is only a function of state, and that we can determine a lower bound on the optimal average reward $\avgrew_{\poltwo^*}(\polone^*)$, denoted by $\hat \avgrew \le \avgrew_{\poltwo^*}(\polone^*)$, as well as the mixing time $\mixtime$. Let us relate the maximum average reward to the influence variables using factors: 
\begin{align*}
\relrangei =  \left ( 1-e^{-\frac{1}{\mixtime}} \right ) \cdot \frac{\hat \avgrew}{2 \cdot \influence_i \cdot \diampoli},
\end{align*} 
which approach infinity as $\influence_i \rightarrow 0$ or $\diampoli \rightarrow 0$. Intuitively, agent $i$ with low influence or low diameter will not be able to negatively influence the obtained reward, implying higher $\relrangei$. The following proposition gives a more exact relationship between the smoothness parameters $(\smoothone, \smoothtwo)$ and factors $\relrangei$. 

\begin{proposition}\label{prop_influence_smooth}
Consider an MDP with a reward function that is only a function of state  (i.e., $\rew(\state,\actionone,\actiontwo) = \rew(\state,\actionone',\actiontwo')$) and the agents' policy spaces with for which  $\relrangetwo \ge \relrangeone > 1$. Then the MDP is $(\smoothone, \smoothtwo)$-smooth with:
\begin{align*}
\smoothone \le \frac{p_1}{p_1 - p_2} \cdot \frac{2 \cdot \relrangetwo - 1}{2\cdot \relrangetwo},\\
\smoothtwo \ge \frac{p_2}{p_1 - p_2} \cdot \frac{2\cdot \relrangeone - 1}{2\cdot \relrangeone - 2},
\end{align*}
where $p_1$ and $p_2$ are free parameters such that $p_1 > p_2 \ge 0$. 
\end{proposition}
\begin{proof}
See below. 
\end{proof}

Condition $\relrangetwo \ge \relrangeone > 1$ simply tells us that \agentone~is more influential than \agenttwo, while the optimal average reward is by a factor greater than the agents' influences. As $\relrangetwo \rightarrow \infty$, we have that the upper bound on $\smoothone$ approaches $\frac{p_1}{p_1 - p_2}$. By setting $p_2 = 0$, we obtain that $\smoothone$ is equal to $1$, which means that \agentone's optimal policy is the same as when she assumes that \agenttwo~ is acting optimally. Hence, \agentone~can technically, with a proper policy, achieve zero regret. In general, $\relrangetwo$ will indicate the degradation in utility, i.e., minimum value of \agentone's regret over the possible choices of \agenttwo's policy. We can similarly analyze other cases. 

\subsection{Proof of Proposition \ref{prop_influence_smooth}}\label{appendix.proof_prop_influence_smooth}

\paragraph{Bound on the differences of average rewards} To prove the proposition we use the following bounds 
on the difference between average rewards: 

\begin{lemma}\label{lm_influence_smooth_bounds}
Consider the policies from Definition \ref{def_smoothness} and assume that reward function is only state dependent. It holds that: 
\begin{align*}
&|\avgrew_{\poltwo}(\polone^*) - \avgrew_{\poltwo^*}(\polone^*)| \le \frac{1}{1-e^{-\frac{1}{\mixtime}}} \cdot  \influencetwo \cdot \diampoltwo,\\
&|\avgrew_{\poltwo}(\polone^*) - \avgrew_{\poltwo}(\polone)| \le \frac{1}{1-e^{-\frac{1}{\mixtime}}} \cdot \influenceone \cdot \diampolone, \\
&| \avgrew_{\poltwo^*}(\polone^*) - \avgrew_{\poltwo}(\polone)| \le \frac{1}{1-e^{-\frac{1}{\mixtime}}} \cdot  (\influenceone \cdot \diampolone +  \influencetwo \cdot \diampoltwo).
\end{align*}
\end{lemma}
\begin{proof} Using the definition of $\avgrew$:
\begin{align*}
&|\avgrew_{\poltwo}(\polone^*) - \avgrew_{\poltwo^*}(\polone^*)|  
\le \left | \statdist_{\polone^*,\poltwo} \cdot \vecdot{ \polone^*, \rews_{\poltwo}}  - \statdist_{\polone^*,\poltwo^*} \cdot \vecdot{ \polone^*, \rews_{\poltwo^*}} \right |\\
&\rightarrow \textit{by the triangle inequality and $+$ and $-$ additional terms}\\
&\le \left |\statdist_{\polone^*,\poltwo} \cdot \vecdot{ \polone^*,\rews_{\poltwo}}  - \statdist_{\polone^*,\poltwo^*} \cdot \vecdot{ \polone^*, \rews_{\poltwo}} \right |
+\left |\statdist_{\polone^*,\poltwo^*} \cdot \vecdot{ \polone^*, \rews_{\poltwo}}  - \statdist_{\polone^*,\poltwo^*} \cdot \vecdot{ \polone^*,\rews_{\poltwo^*}}\right|\\
&\rightarrow \textit{by Holder's inequality}\\
&\le \norm{\statdist_{\polone^*,\poltwo}- \statdist_{\polone^*,\poltwo^*} }_{1} \cdot \norm{\vecdot{ \polone^*, \rews_{\poltwo}}}_{\infty}
+\norm{\statdist_{\polone^*,\poltwo^*} }_{1} \cdot \norm{ \vecdot{ \polone^*, \rews_{\poltwo}}  - \vecdot{ \polone^*, \rews_{\poltwo^*}}}_{\infty}\\
&\rightarrow \textit{(Generalized) Lemma \ref{lm_max_pol_rew} + reward func. is only state dependent, i.e., $\rews_{\poltwo} = \rews^*$}\\
&\le\norm{\statdist_{\polone^*,\poltwo}- \statdist_{\polone^*,\poltwo^*} }_1\\
&\rightarrow \textit{Lemma \ref{lm_dist_bound_1}}\\
&\le \frac{1}{1-e^{-\frac{1}{\mixtime}}} \cdot  \norm{\trans_{\polone^*,\poltwo}- \trans_{\polone^*,\poltwo^*} }_{\infty} \\
&\rightarrow \textit{Bt the definition of influence} \\
&\le \frac{1}{1-e^{-\frac{1}{\mixtime}}} \cdot \influencetwo \cdot \diampoltwo
\end{align*}
The second inequality follows analogously (one can think of it as reversing roles for \agentone~ and \agenttwo). The third is obtained by combining the first two and using the triangle inequality. 
\end{proof}

\paragraph{Proof of the proposition}
\begin{proof}
From Lemma \ref{lm_influence_smooth_bounds}, we have that:
\begin{align*}
&p_1\cdot \avgrew_{\poltwo}(\polone^*) \ge p_1\cdot  \avgrew_{\poltwo^*}(\polone^*) - p_1\cdot  \frac{1}{1-e^{-\frac{1}{\mixtime}}} \cdot \influencetwo \cdot \diampoltwo,\\
&- p_2\cdot \avgrew_{\poltwo}(\polone^*) \ge -p_2\cdot \avgrew_{\poltwo}(\polone) - p_2\cdot \frac{1}{1-e^{-\frac{1}{\mixtime}}} \cdot \influenceone \cdot \diampolone,
\end{align*}
which gives us:
\begin{align*}
(p_1 - p_2) \cdot \avgrew_{\poltwo}(\polone^*) &\ge p_1\cdot  \left (\avgrew_{\poltwo^*}(\polone^* ) - \frac{1}{1-e^{-\frac{1}{\mixtime}}} \cdot \influencetwo \cdot \diampoltwo \right) - p_2 \cdot \left (\avgrew_{\poltwo}(\polone) + \frac{1}{1-e^{-\frac{1}{\mixtime}}} \cdot \influenceone \cdot \diampolone \right )
\end{align*}
Now notice that:
\begin{align*}
-\frac{1}{1-e^{-\frac{1}{\mixtime}}} \cdot \influencetwo \cdot \diampoltwo = - \frac{\hat \avgrew}{2\cdot\relrangetwo} \ge -\frac{\avgrew_{\poltwo^*}(\polone^*)}{2\cdot\relrangetwo}.
\end{align*}
Furthermore, from Lemma \ref{lm_influence_smooth_bounds}, we have: 
\begin{align*}
-\frac{1}{1-e^{-\frac{1}{\mixtime}}} \cdot \influenceone \cdot \diampolone &\ge - \frac{\avgrew_{\poltwo}(\polone)\cdot \frac{1}{1-e^{-\frac{1}{\mixtime}}} \cdot \influenceone \cdot \diampolone}{\avgrew_{\poltwo^*}(\polone^*) - \frac{1}{1-e^{-\frac{1}{\mixtime}}} \cdot \influenceone \cdot \diampolone -  \frac{1}{1-e^{-\frac{1}{\mixtime}}} \cdot \influencetwo \cdot \diampoltwo} \\
&\ge - \frac{\avgrew_{\poltwo}(\polone) \cdot \frac{1}{1-e^{-\frac{1}{\mixtime}}} \cdot \influenceone \cdot \diampolone}{\hat \avgrew - \frac{1}{1-e^{-\frac{1}{\mixtime}}} \cdot\influenceone \cdot \diampolone -  \frac{1}{1-e^{-\frac{1}{\mixtime}}} \cdot\influencetwo \cdot \diampoltwo}\\
&= - \frac{\avgrew_{\poltwo}(\polone) \cdot \frac{1}{\relrangeone}}{2 -\frac{1}{\relrangeone} -  \frac{1}{\relrangetwo}}\\
&\rightarrow \textit{By $\relrangeone \le \relrangetwo$ and $\relrangeone > 1$}\\
&\ge -\frac{\avgrew_{\poltwo^*}(\polone^*)}{2\cdot(\relrangeone - 1)} .
\end{align*}
By putting this together, we obtain:
\begin{align*}
(p_1 - p_2) \cdot \avgrew_{\poltwo}(\polone^*) \ge p_1 \cdot  \avgrew_{\poltwo^*}(\polone^*)\cdot \left (1 -   \frac{1}{2\cdot \relrangetwo} \right )  -p_2\cdot \avgrew_{\poltwo}(\polone)\cdot \left (1 + \frac{1}{2 \cdot (\relrangeone - 1)} \right).
\end{align*}
Therefore, we have that $\smoothone$ and $\smoothtwo$ have to satisfy:
\begin{align*}
\smoothone \le \frac{p_1}{p_1 - p_2} \cdot \frac{2 \cdot \relrangetwo - 1}{2\cdot \relrangetwo},\\
\smoothtwo \ge \frac{p_2}{p_1 - p_2} \cdot \frac{2\cdot \relrangeone - 1}{2\cdot \relrangeone - 2}.
\end{align*}

\end{proof}
     
\clearpage
\section{Proof of Lemma \ref{lemma_large_rounds}}\label{appendix.proof_lemma_large_rounds}

\begin{proof}
Let us first express $\return_{\tixone}$ in terms of $\avgrewt$. We have:
\begin{align*}
&\return_{\tixone} =  \frac{1}{M} \cdot \sum_{\round = 1}^{\rounds} \statdistr  \cdot  \vecdot{\polonet, \rewst} \\
&=  \frac{1}{M} \cdot \sum_{\round = 1}^{\rounds}   \statdistt \cdot \vecdot{ \polonet , \rewst}\\
&+\frac{1}{M} \cdot \sum_{\round = 1}^{\rounds} (\statdistr - \statdistt)   \cdot  \vecdot{\polonet, \rewst} \\
&\ge  \frac{1}{M} \cdot \sum_{\round = 1}^{\rounds}   \statdistt \cdot \vecdot{ \polonet , \rewst}-\left | \frac{1}{M} \cdot \sum_{\round = 1}^{\rounds} (\statdistr - \statdistt)   \cdot  \vecdot{\polonet, \rewst} \right | \\
&\rightarrow \textit{Using the triangle and Holder's inequalities }\\
&\ge  \frac{1}{M} \cdot \sum_{\round = 1}^{\rounds}   \statdistt \cdot \vecdot{ \polonet , \rewst} -\frac{1}{M} \cdot  \sum_{\round = 1}^{\rounds} \norm{\statdistr -\statdistt}_1 \cdot \norm{\vecdot{\polonet, \rewst}}_{\infty} \\
&=  \avgrewt(\polonet) - \frac{1}{M} \cdot \sum_{\round = 1}^{\rounds} \norm{\statdistr -\statdistt}_1.
\end{align*}
Due to Lemma \ref{lm_dist_bound_2}, the summation in the second term is bounded by:
\begin{align*}
& \sum_{\round = 1}^{\rounds} \norm{\statdistr -\statdistt}_1 \le  \norm{\statdist_{\tixone, 1} -\statdistt}_1 \cdot \sum_{\round = 1}^{\rounds} e^{-\frac{\round-1}{\mixtime}} \le \frac{2}{1- e^{-\frac{1}{\mixtime}}},
\end{align*}
Therefore:  
\begin{align*}
\return_{\tixone} \ge \avgrew_{\tixone}(\polone_{\tixone}) - \frac{2}{M\cdot (1- e^{-\frac{1}{\mixtime}})},
\end{align*}
that is:
\begin{align}\label{eq_lower_bound_v}
\bar \return \ge \frac{1}{\tix} \sum_{\tixone = 1}^\tix \avgrew_{\tixone}(\polone_{\tixone}) - \frac{2}{M \cdot (1- e^{-\frac{1}{\mixtime}})}.
\end{align}
Now we follow the analysis of \cite{syrgkanis2015fast} to connect $\sum_{\tixone = 1}^\tix \avgrewt(\polonet)$ to the optimum average reward:
\begin{align*}
 &\sum_{\tixone = 1}^\tix \avgrewt(\polonet) \ge \sum_{\tixone = 1}^\tix \avgrewt(\polone^*) - \regret(\tix) \\
 &\rightarrow \textit{Using the smoothness assumption }\\
 &\ge \sum_{\tixone = 1}^\tix \bracket{ \smoothone \cdot \avgrew_{\poltwo^*}(\polone^*) - \smoothtwo \cdot \avgrewt(\polonet)} - \regret(\tix),
\end{align*}
implying:
\begin{align*}
 &\sum_{\tixone = 1}^\tix \avgrew_{\tixone}(\polone_{\tixone}) \ge \frac{\smoothone}{1 + \smoothtwo} \avgrew_{\poltwo^*}(\polone^*) - \frac{1}{1 + \smoothtwo} \regret(T).
\end{align*}
Now, suppose that $\polone_{\opt}$ and $\poltwo_{\opt}$ are two policies that achieve $\opt$. Using the same approach as for lower-bounding $\return_{\tixone}$, 
we can upper-bound $\opt$ by:
\begin{align*}
\opt \le  \avgrew_{\poltwo_{\opt}}(\polone_{\opt}) + \frac{2}{M\cdot (1- e^{-\frac{1}{\mixtime}})},
\end{align*} 
where $\avgrew_{\poltwo_{\opt}}(\polone_{\opt})$ is the average reward for policies $\polone_{\opt}$ and $\poltwo_{\opt}$. Due to the optimality of $\avgrew_{\poltwo^*}(\polone^*)$, we know that
$\avgrew_{\poltwo^*}(\polone^*) \ge  \avgrew_{\poltwo_{\opt}}(\polone_{\opt})$, which gives us:
\begin{align*}
\opt \le  \avgrew_{\poltwo^*}(\polone^*) + \frac{2}{M\cdot (1- e^{-\frac{1}{\mixtime}})},
\end{align*} 
and further:
\begin{align*}
&\sum_{\tixone = 1}^\tix \avgrew_{\tixone}(\polone_{\tixone}) \ge \frac{\smoothone}{1 + \smoothtwo} \opt - \frac{1}{1 + \smoothtwo} \regret(T) - \frac{\smoothone}{1 + \smoothtwo} \cdot  \frac{2}{M\cdot (1- e^{-\frac{1}{\mixtime}})}.
\end{align*}
Combining this with \eqref{eq_lower_bound_v} we obtain the claim. 
\end{proof}

\clearpage
\section{Useful Lemmas and Proposition for Regret Analysis}\label{appendix.lemmas_for_regret}

 \subsection{Proof of Lemma \ref{lm_oftrl_property_simple}}\label{appendix.proof_lm_oftrl_property}

\begin{proof}
The claim follows from Proposition 7 in \cite{syrgkanis2015fast} (more precisely, Theorem 19 and Lemma 20) by noting that the loss function is 
$\qvals_{\tixone}$ and that the update of $\weights_{\tixone, \tixtwo}(\state)$ corresponds to OFTRL. That is, for each state $\state$ and $1 \le \tixone \le \tix - \Segment + 1$:
\begin{align*}
\sum_{\tixtwo=1}^{\Segment}
\qvals_{\tixone + \tixtwo-1}(\state) \cdot (\polone(\state)  - \weights_{\tixone + \tixtwo-1, \tixtwo}(\state))^\transpose &\le \frac{\Delta_{\regul}}{\ratelearn} + \ratelearn \cdot \sum_{\tixtwo=1}^{\Segment}  \norm{ \qvals_{\tixone + \tixtwo-1}(\state) -  \qvals_{\tixone + \tixtwo - 2}(\state)}_{\infty}^2 - \\
&- \frac{1}{4 \cdot \ratelearn} \cdot \sum_{\tixtwo=1}^{\Segment} \norm{\weights_{\tixone + \tixtwo-1, \tixtwo}(\state) - \weights_{\tixone + \tixtwo-2, \tixtwo-1}(\state)}_{1}^2,
\end{align*}
which implies the statement. (Note that for $\weights(s)$-difference we use $\norm{\cdot}_{1}$, whereas for $\weights$-difference we use $\norm{\cdot}_{\infty}$.)
\end{proof}

We provide a more general version of the lemma that we actually use in the proof of our main result.

\begin{lemma}\label{lm_oftrl_property}
Let $\theta(\tixone) = \min \{ \tau | \tixone + \tau - 1 \ge 1, \tau \ge 1\}$ and $\Theta(\tixone) = \max \{ \tau | \tixone + \tau - 1 \le \tix, \tau \le \Segment  \}$ and $\mathbf 1$ denote column vector of ones with
$|\states|$ elements. Then, 
for each episode $ - \Segment + 2 \le \tixone \le \tix$ we have: 
\begin{align*}
\sum_{\tixtwo=\theta(\tixone)}^{\Theta(\tixone)}
& \vecdot{\polone - \weights_{\tixone + \tixtwo - 1, \tixtwo},  \qvals_{\tixone + \tixtwo - 1}} \le  \\
&\le \mathbf 1 \cdot \left (  \frac{\Delta_{\regul}}{\ratelearn} + \ratelearn \cdot \sum_{\tixtwo=\theta(\tixone)}^{\Theta(\tixone)}  \norm{ \qvals_{\tixone + \tixtwo - 1} -  \qvals_{\tixone + \tixtwo - 2}}_{\max}^2 \right . \\
&\left . - \frac{1}{4 \cdot \ratelearn} \cdot \sum_{\tixtwo=\theta(\tixone)}^{\Theta(\tixone)}  \norm{\weights_{\tixone + \tixtwo - 1, \tixtwo}- \weights_{\tixone + \tixtwo-2, \tixtwo-1}}_{\infty}^2  \right ),
\end{align*}
 where $\Delta_{\regul} = \sup_{\weights \in \simplexone} \regul(\weights) -  \inf_{\weights \in \simplexone} \regul(\weights)$ and
 $\polone$ is an arbitrary policy of \agentone.
\end{lemma}
\begin{proof}
As with Lemma \ref{lm_oftrl_property_simple}, the claim follows from Proposition 7 in \cite{syrgkanis2015fast} (more precisely, Theorem 19 and Lemma 20) by noting that the loss function is 
$\qvals_{\tixone}$ and that the update of $\weights_{\tixtwo, k}(\state)$ corresponds to OFTRL. Note that for $1 \le \tixone \le \tix - \Segment + 1$, Lemma \ref{lm_oftrl_property_simple} yields the claim. We recognize two other cases. First, for $\tixone < 1$, define $\tixone' = |\tixone| + 1 $. We have:
\begin{align*}
&\sum_{\tixtwo=\theta(\tixone)}^{\Theta(\tixone)}
\qvals_{\tixone + \tixtwo - 1}(\state) \cdot (\polone(\state)  - \weights_{\tixone, \tixtwo }(\state))^\transpose = 
\sum_{\tixtwo=1}^{\Segment - \tixone' }
\qvals_{\tixtwo}(\state) \cdot (\polone(\state)  - \weights_{\tixtwo, \tixtwo + \tixone'}(\state))^\transpose \\&\le 
\frac{\Delta_{\regul}}{\ratelearn} + \ratelearn \cdot \sum_{\tixtwo=1}^{\Segment - \tixone}  \norm{ \qvals_{ \tixtwo }(\state) -  \qvals_{\tixtwo - 1}(\state)}_{\infty}^2 -
\frac{1}{4 \cdot \ratelearn} \cdot \sum_{\tixtwo=1}^{\Segment - \tixone}  \norm{\weights_{ \tixtwo, \tixtwo + \tixone'}(\state) - \weights_{ \tixtwo-1, \tixtwo-1 + \tixone'}(\state)}_{1}^2 \\
&=\frac{\Delta_{\regul}}{\ratelearn} + \ratelearn \cdot \sum_{\tixtwo=\theta(\tixone)}^{\Theta(\tixone)}  \norm{ \qvals_{\tixone + \tixtwo - 1}(\state) -  \qvals_{\tixone + \tixtwo - 2}(\state)}_{\infty}^2
- \frac{1}{4 \cdot \ratelearn} \cdot \sum_{\tixtwo=\theta(\tixone)}^{\Theta(\tixone)}  \norm{\weights_{\tixone + \tixtwo - 1, \tixtwo}(\state) - \weights_{\tixone + \tixtwo-2, \tixtwo-1}(\state)}_{1}^2.
\end{align*}
Second, for $\tixone > \tix - \Segment + 1$, we have:
\begin{align*}
&\sum_{\tixtwo=\theta(\tixone)}^{\Theta(\tixone)}
\qvals_{\tixone + \tixtwo - 1}(\state) \cdot (\polone(\state)  - \weights_{\tixone + \tixtwo - 1, \tixtwo - \tixone + 1}(\state))^\transpose =
\sum_{\tixtwo=\tixone}^{\tix}
\qvals_{ \tixtwo}(\state) \cdot (\polone(\state)  - \weights_{ \tixtwo,  \tixtwo - \tixone + 1}(\state))^\transpose \\&\le 
\frac{\Delta_{\regul}}{\ratelearn} + \ratelearn \cdot \sum_{\tixtwo=\tixone}^{\tix}  \norm{ \qvals_{ \tixtwo}(\state) -  \qvals_{ \tixtwo - 1}(\state)}_{\infty}^2 -
\frac{1}{4 \cdot \ratelearn} \cdot \sum_{\tixtwo=\tixone}^{\tix} \norm{\weights_{ \tixtwo,  \tixtwo - \tixone + 1}(\state) - \weights_{ \tixtwo-1,  \tixtwo - \tixone }(\state)}_{1}^2 \\
&=\frac{\Delta_{\regul}}{\ratelearn} + \ratelearn \cdot \sum_{\tixtwo=\theta(\tixone)}^{\Theta(\tixone)}  \norm{ \qvals_{\tixone + \tixtwo - 1}(\state) -  \qvals_{\tixone + \tixtwo - 2}(\state)}_{\infty}^2
- \frac{1}{4 \cdot \ratelearn} \cdot \sum_{\tixtwo=\theta(\tixone)}^{\Theta(\tixone)}  \norm{\weights_{\tixone + \tixtwo - 1,  \tixtwo}(\state) - \weights_{\tixone + \tixtwo-2,  \tixtwo-1}(\state)}_{1}^2.
\end{align*}

Putting everything together, for each state $\state$:
\begin{align*}
\sum_{\tixtwo=\theta(\tixone)}^{\Theta(\tixone)}
\qvals_{\tixone + \tixtwo - 1}(\state) \cdot (\polone(\state)  - \weights_{\tixone + \tixtwo - 1, \tixtwo}(\state))^\transpose &\le \frac{\Delta_{\regul}}{\ratelearn} + \ratelearn \cdot \sum_{\tixtwo=\theta(\tixone)}^{\Theta(\tixone)}  \norm{ \qvals_{\tixone + \tixtwo - 1}(\state) -  \qvals_{\tixone + \tixtwo - 2}(\state)}_{\infty}^2 - \\
&- \frac{1}{4 \cdot \ratelearn} \cdot \sum_{\tixtwo=\theta(\tixone)}^{\Theta(\tixone)}  \norm{\weights_{\tixone + \tixtwo - 1, \tixtwo}(\state) - \weights_{\tixone + \tixtwo-2, \tixtwo-1}(\state)}_{1}^2,
\end{align*}
which implies the statement. (Note that for $\weights(s)$-difference we use $\norm{\cdot}_{1}$, whereas for $\weights$-difference we use $\norm{\cdot}_{\infty}$.)
\end{proof}

\subsection{Proof of Lemma \ref{lm_weight_diff}}\label{appendix.proof_lm_weight_diff}
\begin{proof}
The first claim follows from Lemma 20 in \cite{syrgkanis2015fast} by noting that $\weights_{\tixone, \tixtwo}$ are updated using OFTRL (see also Section 3.2 in  \cite{syrgkanis2015fast}), while $\qval$-values are bounded by $\frac{3}{1-e^{-\frac{1}{\mixtime}}}$ (see Lemma \ref{lm_q_bound}).  In particular from Lemma 20 in \cite{syrgkanis2015fast}, the triangle inequality, and Lemma \ref{lm_q_bound}, it follows:
\begin{align*}
\norm{\weights_{\tixone, \tixtwo}(\state) - \weights_{\tixone-1, \tixtwo-1}(\state)}_{1} &\le \ratelearn \cdot \norm{\qvals_{\tixone-1}(\state) - \qvals_{\tixone}(\state)}_{\infty} + \ratelearn \cdot \norm{\qvals_{\tixone}(\state)}_{\infty} \le 2 \cdot \ratelearn \cdot \norm{\qvals_{\tixone}(\state)}_{\infty} + \ratelearn \cdot \norm{\qvals_{\tixone-1}(\state)}_{\infty} \\
&\le  \ratelearn \cdot  \frac{9}{1-e^{-\frac{1}{\mixtime}}}.
\end{align*}
By taking into account that $\weights_{\tixone, \tixtwo} \in \simplexone$, we know that $\norm{\weights_{\tixone, \tixtwo}(\state) - \weights_{\tixone-1, \tixtwo-1}(\state)}_{1} \le 2$, which together with the above proofs the first claim.

The second claim
follows from the first claim, the triangle inequality, and the fact that $\weights_{\tixone, \tixtwo} \in \simplexone$ (so that $\norm{\weights_{\tixone, \tixtwo_1} - \weights_{\tixone-1, \tixtwo_2} }_{\infty} \le 2$): 
\begin{align*}
 \norm{\polone_{\tixone} - \polone_{\tixone-1}}_{\infty} &= \frac{1}{ \Segment} \norm{ \sum_{\tixtwo = 1}^{\Segment} \weights_{\tixone, \tixtwo} - \weights_{\tixone-1, \tixtwo} }_{\infty} \le \frac{1}{ \Segment} \norm{ \sum_{\tixtwo = 2}^{\Segment} \weights_{\tixone, \tixtwo} - \weights_{\tixone-1, \tixtwo-1} }_{\infty} + \frac{1}{ \Segment} \cdot \norm{\weights_{\tixone, 1} - \weights_{\tixone-1, \Segment} }_{\infty}\\
 &\le  \frac{1}{ \Segment} \sum_{\tixtwo = 2}^{\Segment} \norm{ \weights_{\tixone, \tixtwo}- \weights_{\tixone-1, \tixtwo-1} }_{\infty}  + \frac{2}{ \Segment} \le  \ratelearn \cdot \frac{6}{1-e^{-\frac{1}{\mixtime}}} + \frac{2}{ \Segment} ,
\end{align*}
for $\tixone > 1$. Since $\polone_{\tixone} \in \simplexone$, we have that $\norm{\polone_{\tixone}(\state) - \polone_{\tixone-1}(\state)}_{1} \le 2$, which completes the proof of the second claim.
\end{proof}

\subsection{Proof of Lemma \ref{lm_stat_dist_inequality}}\label{appendix.proof_lm_stat_dist_inequality}

\begin{proof}
By the triangle inequality, we have:
\begin{align*}
&\norm{\statdist_{\polone_{\tixone}, \poltwo_{\tixone}} - \statdist_{\polone_{\tixone-1}, \poltwo_{\tixone-1}} }_1
\le \norm{\statdist_{\polone_{\tixone}, \poltwo_{\tixone}} - \statdist_{\polone_{\tixone}, \poltwo_{\tixone-1}} }_1 
+ \norm{\statdist_{\polone_{\tixone}, \poltwo_{\tixone-1}} - \statdist_{\polone_{\tixone-1}, \poltwo_{\tixone-1}} }_1,
\end{align*}
so we bound each term on the right hand side of the inequality.
Due to Lemma \ref{lm_dist_bound_1} and the definition of the influence, we have:
\begin{align*}
&\norm{\statdist_{\polone_{\tixone}, \poltwo_{\tixone}} - \statdist_{\polone_{\tixone}, \poltwo_{\tixone-1}} }_1 \le \frac{\norm{\trans_{\polone, \poltwo} - \trans_{\polone, \poltwo'}}_{\infty}}{1-e^{-\frac{1}{\mixtime}}} \le  \frac{\influencetwo \cdot \norm{\poltwo_{\tixone}- \poltwo_{\tixone-1}}_{\infty}}{1-e^{-\frac{1}{\mixtime}}}.
\end{align*}
Using the fact that $\norm{\poltwo_{\tixone}- \poltwo_{\tixone-1}}_{\infty} \le \ratechangetwo$, we obtain:
\begin{align*}
 \norm{\statdist_{\polone_{\tixone}, \poltwo_{\tixone}} - \statdist_{\polone_{\tixone}, \poltwo_{\tixone-1}} }_1 \le \frac{\influencetwo \cdot \ratechange_2}{1-e^{-\frac{1}{\mixtime}}}.
\end{align*}
Symmetrically, except that we do not quantify \agentone's influence (i.e., instead use Lemma \ref{lm_trans_bound}):
\begin{align*}
\norm{\statdist_{\polone_{\tixone}, \poltwo_{\tixone-1}} - \statdist_{\polone_{\tixone-1}, \poltwo_{\tixone-1}} }_1 \le \frac{\ratechange_1}{1-e^{-\frac{1}{\mixtime}}}.
\end{align*}
The seconds claim of the lemma follows directly from the analysis above, i.e., from Lemma \ref{lm_dist_bound_1} and the definitions of $\influencetwo$ and $\ratechangetwo$. 
\end{proof}

\subsection{Proof of Lemma \ref{lm_qvals_diff_pol}}\label{appendix.proof_lm_qvals_diff_pol}

\begin{proof}
Notice that:
\begin{align*}
\norm{\qvals_{\tixone}^{\polonet} - \qvals_{\tixone - 1}^{ \polone_{\tixone-1}}}_{\infty} = \max_{\state} |\qvals_{\tixone}^{\polonet}(\state) - \qvals_{\tixone - 1}^{ \polone_{\tixone-1}}(\state)|,
\end{align*}
so it suffices to bound $|\qvals_{\tixone}^{\polonet}(\state) - \qvals_{\tixone - 1}^{ \polone_{\tixone-1}}(\state)|$ for an arbitrary state $\state$.
To calculate $\qvals_{\tixone}^{\polonet}(\state) - \qvals_{\tixone - 1}^{ \polone_{\tixone-1}}(\state)$, 
denote a row vector of ones with $|\states|$ elements by $\mathbf 1$ and set the initial state distribution to $\statdist_{1}(\state') = \mathbbm 1_{\state = \state'}$ 
(i.e., the initial state is $\state$). 
As shown in the proof of Lemma 3 of \cite{even2005experts}, 
$\qvals_{\tixone}(\state, \action)$ can be represented as an infinite time series that converges in absolute values. 
This implies:
\begin{align*}
&\qvals_{\tixone}^{\polonet}(\state) - \qvals_{\tixone - 1}^{ \polone_{\tixone-1}}(\state) 
= \sum_{\round = 1}^{\infty} \left ( \statdistr \cdot \vecdot{\polonet, \rewst} - \mathbf 1 \cdot \avgrewt(\polonet) \right ) 
-\sum_{m = 1}^{\infty} \left (\statdist_{\tixone-1, \round}  \vecdot{ \polone_{\tixone-1},  \rews_{\tixone-1}} - \mathbf 1 \cdot \avgrew_{\tixone-1}(\polone_{\tixone-1}) \right )\\
&= \sum_{\round = 1}^{\infty} \Big (  \statdistt \cdot \vecdot{\polonet, \rewst} - \mathbf 1 \cdot \avgrewt(\polonet) \Big . 
-  \statdist_{\polone_{\tixone-1}, \poltwo_{\tixone-1}} \cdot \vecdot{ \polone_{\tixone-1},  \rews_{\tixone-1}} + \mathbf 1 \cdot \avgrew_{\tixone-1}(\polone_{\tixone-1}) \\
&+(\statdistr - \statdistt ) \cdot \vecdot{\polonet, \rewst} 
- \Big . (\statdist_{\tixone - 1, \round } - \statdist_{\polone_{\tixone-1}, \poltwo_{\tixone-1}} ) \cdot \vecdot{ \polone_{\tixone-1},  \rews_{\tixone-1}}  \Big )  \\
&\rightarrow \textit{ By the definition of avg. rev. }\\
&= \sum_{\round = 1}^{\infty} \Big ( (\statdistr - \statdistt ) \cdot \vecdot{\polonet, \rewst}  - (\statdist_{\tixone - 1, \round } - \statdist_{\polone_{\tixone-1}, \poltwo_{\tixone-1}} ) \cdot  \vecdot{ \polone_{\tixone-1},  \rews_{\tixone-1}} \Big ) \\
&\rightarrow \textit{ By rearranging, $+$ and $-$ additional terms}\\
&= \sum_{\round = 1}^{\infty} ((\statdist_{\tixone, \round } - \statdist_{\polone_{\tixone}, \poltwo_{\tixone}} ) - (\statdist_{\tixone-1, \round } - \statdist_{\polone_{\tixone-1}, \poltwo_{\tixone-1}} ))  \cdot \vecdot{\polonet, \rewst}\\
&+ \sum_{\round = 1}^{\infty} (\statdist_{\tixone-1, \round } - \statdist_{\polone_{\tixone-1}, \poltwo_{\tixone-1}} ) \cdot (\vecdot{\polonet, \rewst} -  \vecdot{ \polone_{\tixone},  \rews_{\tixone-1}})\\
&+ \sum_{\round = 1}^{\infty} (\statdist_{\tixone-1, \round } - \statdist_{\polone_{\tixone-1}, \poltwo_{\tixone-1}} ) \cdot (\vecdot{ \polone_{\tixone},  \rews_{\tixone-1}} - \vecdot{ \polone_{\tixone-1},  \rews_{\tixone-1}})\\
\end{align*}
Using the triangle and Holder's inequalities, we obtain:
\begin{align*}
& |\qvals_{\tixone}^{\polonet}(\state) - \qvals_{\tixone - 1}^{ \polone_{\tixone-1}}(\state) |\le \\
&=\sum_{\round = 1}^{\infty} \norm{(\statdist_{\tixone, \round} - \statdist_{\polone_{\tixone}, \poltwo_{\tixone}} ) - ( \statdist_{\tixone-1, \round}  - \statdist_{\polone_{\tixone-1}, \poltwo_{\tixone-1}} )}_{1}  \cdot  \norm{\vecdot{\polonet, \rewst}}_{\infty}\\
&+\sum_{\round = 1}^{\infty} \norm{\statdist_{\tixone-1, \round } - \statdist_{\polone_{\tixone-1}, \poltwo_{\tixone-1}}}_{1} \cdot \norm{\vecdot{\polonet, \rewst} -  \vecdot{ \polone_{\tixone},  \rews_{\tixone-1}}}_{\infty}\\
&+\sum_{\round = 1}^{\infty} \norm{\statdist_{\tixone-1, \round } - \statdist_{\polone_{\tixone-1}, \poltwo_{\tixone-1}} }_{1} \cdot \norm{\vecdot{ \polone_{\tixone},  \rews_{\tixone-1}} - \vecdot{ \polone_{\tixone-1},  \rews_{\tixone-1}}}_{\infty}\\
&\rightarrow \textit{ By Lemma  \ref{lm_max_pol_rew} and Lemma \ref{lm_diff_pol_rew}  }\\
&\le \sum_{\round = 1}^{\infty} \norm{(\statdist_{\tixone, \round} - \statdist_{\polone_{\tixone}, \poltwo_{\tixone}} ) - ( \statdist_{\tixone-1, \round}  - \statdist_{\polone_{\tixone-1}, \poltwo_{\tixone-1}} )}_{1}  +\ratechangetwo \cdot \sum_{\round = 1}^{\infty} \norm{\statdist_{\tixone-1, \round } - \statdist_{\polone_{\tixone-1}, \poltwo_{\tixone-1}}}_{1} \\
&+\ratechangeone \cdot \sum_{\round = 1}^{\infty} \norm{\statdist_{\tixone-1, \round } - \statdist_{\polone_{\tixone-1}, \poltwo_{\tixone-1}} }_{1} \\
&\rightarrow \textit{ By Lemma \ref{lm_dist_bound_4} and Lemma \ref{lm_dist_bound_2} }\\
&\le 3 \cdot \frac{\ratechangeone +  \influencetwo \cdot \ratechangetwo}{(1-e^{-\frac{1}{\mixtime}})^2} +
+2\cdot (\ratechangeone + \ratechangetwo) \cdot \sum_{\round = 1}^{\infty} e^{-\frac{\round-1}{\mixtime}}\\
&\le 3 \cdot \frac{\ratechangeone +  \influencetwo \cdot \ratechangetwo}{(1-e^{-\frac{1}{\mixtime}})^2} + 2 \cdot \frac{\ratechangeone+ \ratechangetwo}{1-e^{-\frac{1}{\mixtime}}}
\end{align*}
which completes the proof. 
\end{proof}

\subsection{Proof of Lemma \ref{lm_qvals_diff_t}}\label{appendix.proof_lm_qvals_diff_t}

\begin{proof}
Using the recursive definition of $\qval$-values and the triangle inequality, we obtain:
\begin{align*}
&\norm{ \qvalst -  \qvals_{\tixone-1}}_{\max} \le \norm{\rewst - \rews_{\tixone-1}}_{\max}
+  |\avgrewt(\polonet) - \avgrew_{\tixone}(\polone_{\tixone-1})| 
&+ \max_{\state, \actionone} | \trans_{\poltwo_{\tixone}}(\state, \actionone) \cdot \qvals_{\tixone}^\polonet -  \trans_{\poltwo_{\tixone-1}}(\state, \actionone,) \cdot \qvals_{\tixone-1}^{\polone_{\tixone-1}} |.
\end{align*}

The bound of the first term is given in Lemma \ref{lm_diff_pol_rew}. For the second term, we have: 
\begin{align*}
&|\avgrewt(\polonet) - \avgrew_{\tixone}(\polone_{\tixone-1})| = |\statdistt \cdot \vecdot{ \polonet, \rewst} -  \statdist_{\polone_{\tixone - 1}, \poltwo_{\tixone - 1}} \cdot \vecdot{ \polone_{\tixone-1}, \rews_{\tixone-1}}|\\
&\rightarrow \textit{ $+$ and $-$ additional terms and the triangle inequality } \\
&\le |(\statdistt - \statdist_{\polone_{\tixone - 1}, \poltwo_{\tixone - 1}})\cdot \vecdot{ \polonet, \rewst} |
+  |\statdist_{\polone_{\tixone - 1}, \poltwo_{\tixone - 1}} \cdot \vecdot{ \polone_{\tixone} - \polone_{\tixone-1}, \rews_{\tixone}}|
+  |\statdist_{\polone_{\tixone - 1}, \poltwo_{\tixone - 1}} \cdot \vecdot{ \polone_{\tixone-1},  \rews_{\tixone} - \rews_{\tixone-1}}|\\
&\rightarrow \textit{By Holder's inequality and bound. $\statdist$, $\rews$, $\pol$ } \\
&\le \norm{\statdistt - \statdist_{\polone_{\tixone - 1}, \poltwo_{\tixone - 1}}}_1 \norm{\vecdot{ \polonet, \rewst}}_{\infty} 
+ \norm{\statdist_{\polone_{\tixone - 1}, \poltwo_{\tixone - 1}}}_{1} \cdot \norm{ \vecdot{ \polone_{\tixone} - \polone_{\tixone-1}, \rews_{\tixone}}}_\infty \\
& +\norm{\statdist_{\polone_{\tixone - 1}, \poltwo_{\tixone - 1}}}_{1} \cdot \norm{ \vecdot{ \polone_{\tixone-1},  \rews_{\tixone} - \rews_{\tixone-1}}}_\infty \\
&\rightarrow \textit{By Lemma \ref{lm_max_pol_rew} and and $\norm{\statdist}_{1} = 1$ } \\
& \le \norm{\statdistt - \statdist_{\polone_{\tixone - 1}, \poltwo_{\tixone - 1}}}_1 + \norm{ \vecdot{ \polone_{\tixone} - \polone_{\tixone-1}, \rews_{\tixone}}}_\infty + \norm{ \vecdot{ \polone_{\tixone-1},  \rews_{\tixone} - \rews_{\tixone-1}}}_\infty \\
&\rightarrow \textit{By Lemma \ref{lm_stat_dist_inequality} and Lemma \ref{lm_diff_pol_rew} } \\
&\le  \frac{\ratechange_1 +  \influencetwo \cdot \ratechange_2}{1-e^{-\frac{1}{\mixtime}}} + \ratechangeone + \ratechangetwo
\end{align*}

Finally, the third term is bounded by:
\begin{align*}
&\max_{\state, \actionone} | \trans_{\poltwo_{\tixone}}(\state, \actionone,) \cdot \qvals_{\tixone}^\polonet -  \trans_{\poltwo_{\tixone-1}}(\state, \actionone) \cdot \qvals_{\tixone-1}^{\polone_{\tixone-1}} | \\
&\rightarrow \textit{ $+$ and $-$ additional terms and the triangle inequality } \\
&\le \max_{\state, \actionone} | \trans_{\poltwo_{\tixone}}(\state, \actionone,) \cdot (\qvals_{\tixone}^\polonet - \qvals_{\tixone-1}^{\polone_{\tixone-1}}) | +\max_{\state, \actionone} |( \trans_{\poltwo_{\tixone}}(\state, \actionone)  -  \trans_{\poltwo_{\tixone-1}}(\state, \actionone) ) \cdot \qvals_{\tixone-1}^{\polone_{\tixone-1}} | \\
&\rightarrow \textit{ Using Holder's inequality } \\
& \le \max_{\state, \actionone} \norm{ \trans_{\poltwo_{\tixone}}(\state, \actionone)}_{1} \cdot \norm{\qvals_{\tixone}^\polonet - \qvals_{\tixone-1}^{\polone_{\tixone-1}}}_\infty  + \max_{\state, \actionone} \norm{ \trans_{\poltwo_{\tixone}}(\state, \actionone)  -  \trans_{\poltwo_{\tixone-1}}(\state, \actionone) }_1 \cdot \norm{\qvals_{\tixone-1}^{\polone_{\tixone-1}} }_\infty \\
&\rightarrow \textit{ By Lemma \ref{lm_trans_bound_2} and $\norm{\trans_{\poltwo}(\state, \actionone)}_{1} = 1$} \\
& \le \norm{\qvals_{\tixone}^\polonet - \qvals_{\tixone-1}^{\polone_{\tixone-1}}}_\infty  +\max_{\state} \norm{\poltwo_{\tixone}(\state)  - \poltwo_{\tixone-1}(\state)}_1 \cdot \norm{\qvals_{\tixone-1}^{\polone_{\tixone-1}} }_\infty \\
& = \norm{\qvals_{\tixone}^\polonet - \qvals_{\tixone-1}^{\polone_{\tixone-1}}}_\infty + \norm{\poltwo   - \poltwo_{\tixone-1}}_{\infty} \cdot \norm{\qvals_{\tixone-1}^{\polone_{\tixone-1}} }_\infty \\
&\rightarrow \textit{By Lemma \ref{lm_qvals_diff_pol} and Lemma \ref{lm_q_bound}} \\
&\le C_{Q^{\pol}} + \frac{3}{1 - e^{\frac{1}{ \mixtime}}} \cdot \ratechangeone.
\end{align*}
By summing the three bounds, we obtain the claim.
\end{proof}

\clearpage
\section{Proof of Theorem \ref{thm_regret}}\label{appendix.proof_thm_regret}

\begin{proof}
Lemma \ref{lm_eta_q_relation} provides the following useful identity:
\begin{align*}
\sum_{\tixone = 1}^{\tix} \avgrewt(\polone)-\avgrewt(\polonet) = \sum_{\tixone = 1}^{\tix} \statdist_{\polone, \poltwo_{\tixone}} \cdot ( \qvalst^\polone  - \qvalst^\polonet ).
\end{align*}
Furthermore, by the definition of $\qvalst^\polone$ and $\qvalst^\polonet$, we have:
\begin{align*}
\qvalst^\polone  -\qvalst^\polonet =\vecdot{\polone - \polonet, \qvals_{\tixone}}. 
\end{align*}
Using the two identities, gives:
\begin{align*}
&\sum_{\tixone = 1}^{\tix} \avgrewt(\polone)-\avgrewt(\polonet)= \sum_{\tixone = 1}^{\tix} \statdist_{\polone, \poltwo_{\tixone}} \cdot \vecdot{\polone - \polonet, \qvals_{\tixone}}
=  \sum_{\tixone = 1}^{\tix} \statdist_{\polone, \poltwo_{\tixone}} \cdot \vecdot{\polone - \frac{1}{\Segment}\sum_{\tixtwo = 1}^{\Segment} \weights_{\tixone, \tixtwo}, \qvals_{\tixone} } \\
&=  \frac{1}{\Segment} \cdot \statdist_{\polone, \poltwo_{1}} \cdot \vecdot{\polone - \weights_{1, \Segment}, \qvals_{1} } \\
&+ \frac{1}{\Segment} \cdot \left ( \statdist_{\polone, \poltwo_{1}} \cdot \vecdot{\polone - \weights_{1, \Segment - 1}, \qvals_{1} } + \statdist_{\polone, \poltwo_{2}} \cdot \vecdot{\polone - \weights_{2, \Segment}, \qvals_{2} }  \right ) \\
&+\cdots \\
& + \sum_{\tixone = 1}^{\tix - \Segment } \sum_{\tixtwo=1}^{\Segment} \statdist_{\polone, \poltwo_{\tixone+\tixtwo - 1}} \cdot \vecdot{ \polone - \weights_{\tixone + \tixtwo - 1, \tixtwo}, \qvals_{\tixone+\tixtwo - 1}} \\
&+ \cdots \\
& + \frac{1}{\Segment} \cdot \left ( \statdist_{\polone, \poltwo_{\tix - 1}} \cdot \vecdot{\polone - \weights_{\tix - 1, 1}, \qvals_{\tix - 1} } + \statdist_{\polone, \poltwo_{\tix}} \cdot \vecdot{\polone - \weights_{\tix, 2}, \qvals_{\tix} }  \right ) \\
& + \frac{1}{\Segment} \cdot \statdist_{\polone, \poltwo_{\tix}} \cdot \vecdot{\polone - \weights_{\tix, 1}, \qvals_{\tix} } \\
&= \frac{1}{\Segment} \cdot \sum_{\tixone = -\Segment+2}^{\tix} \sum_{\tixtwo=\theta(\tixone)}^{\Theta(\tixone)} \statdist_{\polone, \poltwo_{\tixone+\tixtwo - 1}} \cdot \vecdot{ \polone - \weights_{\tixone + \tixtwo - 1, \tixtwo}, \qvals_{\tixone+\tixtwo - 1}}
\end{align*}
where we introduced $\theta(\tixone) = \min \{ \tau | \tixone + \tau - 1 \ge 1, \tau \ge 1\}$ and $\Theta(\tixone) = \max \{ \tau | \tixone + \tau - 1 \le \tix, \tau \le \Segment \}$ as in Lemma \ref{lm_oftrl_property}. 
Now, by replacing $\statdist_{\polone, \poltwo_{\tixone + \tixtwo-1}}$ with $\statdist_{\polone, \poltwo_{\tixone + \theta(\tixone)-1}}$, it follows that  the above is equal to:
\begin{align*}
&\frac{1}{\Segment} \cdot \sum_{\tixone = -\Segment+2}^{\tix} \sum_{\tixtwo=\theta(\tixone)}^{\Theta(\tixone)} \statdist_{\polone, \poltwo_{\tixone +\theta(\tixone)-1}} \cdot \vecdot{ \polone - \weights_{\tixone + \tixtwo - 1 , \tixtwo}, \qvals_{\tixone+\tixtwo - 1}}\\
&+ \frac{1}{\Segment} \cdot \sum_{\tixone = -\Segment+2}^{\tix} \sum_{\tixtwo=\theta(\tixone)}^{\Theta(\tixone)}( \statdist_{\polone, \poltwo_{\tixone+\tixtwo-1}} - \statdist_{\polone, \poltwo_{\tixone +\theta(\tixone)-1}}) 
\cdot \vecdot{ \polone - \weights_{\tixone + \tixtwo - 1, \tixtwo}, \qvals_{\tixone+\tixtwo - 1}}
\end{align*}
We bound the two terms separately, starting with the first term. Since $\statdist_{\polone, \poltwo_{\theta(\tixone)}}$ is independent of $\tixtwo$, the first term is equal to:
\begin{align*}
& \frac{1}{\Segment} \cdot \sum_{\tixone = -\Segment+2}^{\tix} \sum_{\tixtwo=\theta(\tixone)}^{\Theta(\tixone)} \statdist_{\polone, \poltwo_{\tixone +\theta(\tixone)-1}} \cdot \vecdot{ \polone - \weights_{\tixone + \tixtwo - 1, \tixtwo}, \qvals_{\tixone+\tixtwo - 1}} 
=  \frac{1}{\Segment} \cdot \sum_{\tixone = -\Segment+2}^{\tix} \statdist_{\polone, \poltwo_{\tixone +\theta(\tixone)-1}} \cdot \sum_{\tixtwo=\theta(\tixone)}^{\Theta(\tixone)} \vecdot{ \polone - \weights_{\tixone + \tixtwo - 1, \tixtwo}, \qvals_{\tixone+\tixtwo - 1}},
\end{align*}
which is by Lemma \ref{lm_oftrl_property} bounded by:
\begin{align*}
&\frac{1}{\Segment} \cdot \sum_{\tixone = -\Segment+2}^{\tix}  \statdist_{\polone, \poltwo_{\tixone +\theta(\tixone)-1}} \cdot \mathbf 1 \cdot  \left (  \frac{\Delta_{\regul}}{\ratelearn}  + \right . 
 \ratelearn \cdot \sum_{\tixtwo=\theta(\tixone)}^{\Theta(\tixone)}  \norm{ \qvals_{\tixone + \tixtwo - 1} -  \qvals_{\tixone + \tixtwo - 2}}_{\max}^2  
\left . - \frac{1}{4 \cdot \ratelearn} \cdot \sum_{\tixtwo=\theta(\tixone)}^{\Theta(\tixone)}  \norm{\weights_{\tixone + \tixtwo - 1, \tixtwo} - \weights_{\tixone + \tixtwo-2, \tixtwo-1}}_{\infty}^2  \right )
\\&
\le \frac{1}{\Segment} \cdot \sum_{\tixone = -\Segment+2}^{\tix}  \statdist_{\polone, \poltwo_{\tixone +\theta(\tixone)-1}} \cdot \mathbf 1 \cdot  \left (  \frac{\Delta_{\regul}}{\ratelearn} +
 \ratelearn \cdot \sum_{\tixtwo=\theta(\tixone)}^{\Theta(\tixone)}  \norm{ \qvals_{\tixone + \tixtwo - 1} -  \qvals_{\tixone + \tixtwo - 2}}_{\max}^2   \right )
\\&
\le \frac{1}{\Segment} \cdot \sum_{\tixone = -\Segment+2}^{\tix}  \left (  \frac{\Delta_{\regul}}{\ratelearn} 
+ \ratelearn \cdot \sum_{\tixtwo=\theta(\tixone)}^{\Theta(\tixone)}  \norm{ \qvals_{\tixone + \tixtwo - 1} -  \qvals_{\tixone + \tixtwo - 2}}_{\max}^2  \right ).
\end{align*}
Proposition \ref{prop_C_bound}, Lemma \ref{lm_qvals_diff_t} and Lemma \ref{lm_weight_diff} imply that the above
is further bounded by:
\begin{align*}
&\frac{1}{\Segment} \cdot \sum_{\tixone = -\Segment+2}^{\tix}  \left (  \frac{\Delta_{\regul}}{\ratelearn} 
+ \ratelearn \cdot \sum_{\tixtwo=\theta(\tixone)}^{\Theta(\tixone)}  \norm{ \qvals_{\tixone + \tixtwo - 1} -  \qvals_{\tixone + \tixtwo - 2}}_{\max}^2  \right )\\
&\rightarrow \textit{Using Lemma \ref{lm_qvals_diff_t} and $\Theta(\tixone) - \theta(\tixone) \le \Segment$} \\ 
 &\le \frac{\tix + \Segment}{\Segment} \cdot  \left (\frac{\Delta_{\regul}}{\ratelearn} + \Segment \cdot \ratelearn \cdot C_{\qval}^2 \right )  \\
 &\rightarrow \textit{By the definition of $C_{\mixtime}$} \\
 &\le \frac{\tix + \Segment}{\Segment} \cdot  \left (\frac{\Delta_{\regul}}{\ratelearn} + \Segment \cdot \ratelearn \cdot C_{\mixtime}^2  \cdot \max \{ \ratechangeone^2, \ratechangetwo^2 \} \right )  \\
 &\rightarrow \textit{Using $k \cdot \ratelearn \ge \ratechangeone$ and $k \cdot \ratelearn \ge \ratechangetwo$, and that $\Segment \le \tix$} \\
 &\le \frac{2 \cdot \tix}{\Segment} \cdot  \left (\frac{\Delta_{\regul}}{\ratelearn} + k^2 \cdot \Segment \cdot \ratelearn^3 \cdot C_{\mixtime}^2  \right ).
 \end{align*}
Notice that in the above analysis $\mathbf 1$ denoted a row vector of ones with $|\states|$ elements.
Finally, setting $\ratelearn = \frac{1}{\Segment^{\frac{1}{4}}}$ leads to the upper bound:
\begin{align*}
\frac{2 \cdot \tix}{\Segment} \left (\frac{\Delta_{\regul}}{\ratelearn} + k^2 \cdot \Segment \cdot \ratelearn^3  \cdot  C_{\mixtime}^2 \right ) 
&= \frac{2 \cdot \tix}{\Segment} \left (\Delta_{\regul} \cdot \Segment^{\frac{1}{4}} + k^2 \cdot \Segment^{\frac{1}{4}} \cdot \ratelearn^3  \cdot  C_{\mixtime}^2 \right )
\\&= 2 \cdot \left (\Delta_{\regul}  +  k^2 \cdot C_{\mixtime}^2 \right ) \cdot T \cdot \Segment^{-\frac{3}{4}}.
\end{align*}

The second term is bounded by its absolute value, which together with the triangle inequality and Holder's inequality, results in: 
\begin{align*}
&\frac{1}{\Segment} \cdot \sum_{\tixone = -\Segment+2}^{\tix} \sum_{\tixtwo=\theta(\tixone)}^{\Theta(\tixone)} ( \statdist_{\polone, \poltwo_{\tixone+\tixtwo-1}} - \statdist_{\polone, \poltwo_{\tixone +\theta(\tixone)-1}}) 
\cdot \vecdot{ \polone - \weights_{\tixone + \tixtwo - 1, \tixtwo}, \qvals_{\tixone+\tixtwo -1}}\\&
\le \frac{1}{\Segment} \cdot \sum_{\tixone = -\Segment+2}^{\tix} \sum_{\tixtwo=\theta(\tixone)}^{\Theta(\tixone)} \norm{\statdist_{\polone, \poltwo_{\tixone+\tixtwo-1}} - \statdist_{\polone, \poltwo_{\tixone +\theta(\tixone)-1}}}_{1}  
 \cdot \norm{\vecdot{ \polone - \weights_{\tixone + \tixtwo - 1, \tixtwo}, \qvals_{\tixone+\tixtwo - 1}} }_{\infty}.
\end{align*}
From Lemma \ref{lm_diff_pol_rew_general} and Lemma \ref{lm_q_bound} it follows that 
$\norm{\vecdot{ \polone - \weights_{\tixone + \tixtwo-1, \tixtwo}, \qvals_{\tixone+\tixtwo-1}} }_{\infty} \le \frac{6}{1 - e^{-\frac{1}{\mixtime}}}$, which leads to the upper bound:  
\begin{align*}
&\frac{1}{\Segment} \cdot \sum_{\tixone = -\Segment+2}^{\tix}\sum_{\tixtwo=\theta(\tixone)}^{\Theta(\tixone)}\norm{\statdist_{\polone, \poltwo_{\tixone+\tixtwo-1}} - \statdist_{\polone, \poltwo_{\tixone +\theta(\tixone)-1}}}_{1} \cdot \frac{6}{1 - e^{-\frac{1}{\mixtime}}},
\end{align*}
which by identity:
\begin{align*}
\statdist_{\polone, \poltwo_{\tixone+\tixtwo-1}} - \statdist_{\polone, \poltwo_{\tixone +\theta(\tixone)-1}} &= \statdist_{\polone, \poltwo_{\tixone+\tixtwo-1}} - \statdist_{\polone, \poltwo_{\tixone+\tixtwo - 2}} \\
&+ \statdist_{\polone, \poltwo_{\tixone+\tixtwo - 2}} - \statdist_{\polone, \poltwo_{\tixone+\tixtwo - 3}} \\
&...\\
&+ \statdist_{\polone, \poltwo_{\tixone +\theta(\tixone) + 1}} -  \statdist_{\polone, \poltwo_{\tixone +\theta(\tixone)}}\\
&+ \statdist_{\polone, \poltwo_{\tixone +\theta(\tixone)}} -  \statdist_{\polone, \poltwo_{\tixone +\theta(\tixone)-1}}\\
&= \sum_{k = \theta(\tixone)}^{\tixtwo-1} (\statdist_{\polone, \poltwo_{\tixone+k}} - \statdist_{\polone, \poltwo_{\tixone+k - 1}}),
\end{align*}
and the triangle inequality is bounded by:
\begin{align*}
&\frac{1}{\Segment} \cdot \frac{6}{1 - e^{-\frac{1}{\mixtime}}} \cdot  \sum_{\tixone = -\Segment+2}^{\tix} \sum_{\tixtwo=\theta(\tixone)}^{\Theta(\tixone)} \sum_{k = \theta(\tixone)}^{\tixtwo-1} \norm{\statdist_{\polone, \poltwo_{\tixone+k}} - \statdist_{\polone, \poltwo_{\tixone+k - 1}}}_{1}.
\end{align*}
Application of Lemma \ref{lm_stat_dist_inequality} additionally bounds the last term by: 
\begin{align*}
\frac{1}{\Segment} \cdot \frac{6}{1 - e^{-\frac{1}{\mixtime}}} \cdot  \sum_{\tixone = -\Segment+2}^{\tix} \sum_{\tixtwo=\theta(\tixone)}^{\Theta(\tixone)} \sum_{k = \theta(\tixone)}^{\tixtwo-1}  \frac{\influencetwo \cdot \ratechange_2}{1-e^{-\frac{1}{\mixtime}}} 
&= \frac{1}{\Segment} \cdot \frac{6}{1 - e^{-\frac{1}{\mixtime}}} \cdot \frac{\influencetwo \cdot \ratechange_2}{1-e^{-\frac{1}{\mixtime}}}  \sum_{\tixone = -\Segment+2}^{\tix} \sum_{\tixtwo=\theta(\tixone)}^{\Theta(\tixone)} (\tixtwo - 1) \\
&\le \frac{1}{\Segment} \cdot \frac{6 \cdot \influencetwo \cdot \ratechange_2}{(1-e^{-\frac{1}{\mixtime}})^2}  \sum_{\tixone = -\Segment+2}^{\tix} \frac{\Segment^2}{2} \\
&\le \frac{1}{\Segment} \cdot (\tix + \Segment) \cdot \Segment^2 \cdot \frac{3 \cdot \influencetwo \cdot \ratechange_2}{(1-e^{-\frac{1}{\mixtime}})^2} \\&=  2 \cdot \tix  \cdot \Segment \cdot \frac{3 \cdot \influencetwo \cdot \ratechange_2}{(1-e^{-\frac{1}{\mixtime}})^2}
\end{align*}

Putting everything together, we obtain:
\begin{align*}
&\regret(\tix) = \sum_{\tixone = 1}^{\tix} \avgrewt(\polone)-\avgrewt(\polonet)  
\le 2 \cdot \left (\Delta_{\regul}  +  k^2 \cdot C_{\mixtime}^2 \right ) \cdot T \cdot \Segment^{-\frac{3}{4}} + \frac{ 6 \cdot \influencetwo \cdot \ratechange_2}{(1-e^{-\frac{1}{\mixtime}})^2} \cdot \tix  \cdot \Segment.
\end{align*}
\end{proof}

\clearpage
\section{Properties of experts with periodic restarts}\label{appendix.properties_algo_oftrl_rand_restart}

\begin{algorithm2e}[t!]
\SetKwBlock{When}{when}{endwhen} 

\textbf{Input:} {Segment horizon $\Segment$, learning rate $\ratelearn$;}

\Begin{

\textbf{Initialize:} $\forall s$, $\polone_{1}(s) = \argmax_{\weights \in \simplexone} \frac{\regul(\weights)}{\ratelearn}$\;

\For{episode $\tixone \in \{ 1,..., \tix \}$}{
	
	Commit to policy $\polone_{\tixone}$\;
	Obtain the return $\return_{\tixone}$\;
	Observe \agenttwo's policy $\poltwo_\tixone$\;
	Calculate Q-values $\qvals_\tixone$\;
	\tcp{Updating policy $\polone$}\
	\tcp{Reset the policy if segment ended}\
	\If{$\tixone \mod \Segment = 0$}{
	    \For{state $\state \in \mathcal S$}{
		$\polone_{t+1}(\state) = \argmax_{\weights \in \simplexone} \frac{\regul(\weights)}{\epsilon}$\;
		}
	}
	\tcp{Otherwise, calculate weights using a cropped horizon}\
	\Else{
	\For{state $\state \in \mathcal S$}{
		$\tixtwo = \floor{\frac{\tixone}{\Segment}}$\;
		$\polone_{\tixone+1}(\state)= \argmax_{\weights \in \simplexone}{ \big (\sum_{k=\tixtwo + 1}^{\tixone} \qvals_{k}(\state) }$\
		$+ \qvals_{\tixone}(\state)\big) \weights^{\transpose} + \frac{\regul(\weights)}{\epsilon}$\;
	}
	}
}
}
\caption{\algexprestart~(Experts with periodic restarts)}\label{algo_oftrl_rand_restart}
\end{algorithm2e}

To prove the property of \algexprestart~that corresponds to Theorem \ref{thm_regret}, we need to change the claims of Lemma \ref{lm_oftrl_property_simple} and Lemma \ref{lm_weight_diff}.
Lemma \ref{lm_stat_dist_inequality}, Lemma \ref{lm_qvals_diff_pol}, Lemma \ref{lm_qvals_diff_t}, and Proposition \ref{prop_C_bound} hold for \algexprestart~(in general, but also for each segment of length $\Segment$ separately). Therefore, we do not restate the claims of the latter results, but simply refer to them. Notice that, after each segment in \algexprestart, we should associate the next episode as if it was the first one. For example, the policy at $\tixone = \Segment + 1$ is not based on historical data, so the algorithmic step that precedes is effectively $\tixone = 0$, with a priori defined $\qvals_0$ and $\polone_0$.  
For notational convenience, we use $\%$ to denote the $\mod$ operation. Furthermore, we will denote the policy change magnitude of \agentone~within one segment by $\ratechangeone^{\Segment}$ (note that $\ratechangeone$ is generally greater than $\ratechangeone^{\Segment}$ due to potentially abrupt changes in \agentone's policy between two segments).  

\begin{lemma}\label{lm_oftrl_property_restart}
Consider \algexprestart~and let  $\mathbf 1$ denote column vector of ones with
$|\states|$ elements. Then, 
for each  $ l \in \{0, 1, ..., \floor{ \frac{\tix}{\Segment} } \}$ we have: 
\begin{align*}
\sum_{\tixone=l\cdot \Segment + 1}^{(l+1)\cdot \Segment}
& \vecdot{\polone - \polone_{\tixone},  \qvals_{\tixone}} \le  \mathbf 1 \cdot \left (  \frac{\Delta_{\regul}}{\ratelearn} + \ratelearn \cdot \sum_{\tixone=l\cdot \Segment + 1}^{(l+1)\cdot \Segment} \norm{ \qvals_{\tixone} -  \qvals_{(\tixone -1)\%\Segment}}_{\max}^2   - \frac{1}{4 \cdot \ratelearn} \cdot \sum_{\tixone=l\cdot \Segment + 1}^{(l+1)\cdot \Segment}  \norm{\polone_{\tixone }- \polone_{(\tixone -1)\%\Segment}}_{\infty}^2  \right ),
\end{align*}
 where $\Delta_{\regul} = \sup_{\weights \in \simplexone} \regul(\weights) -  \inf_{\weights \in \simplexone} \regul(\weights)$ and
 $\polone$ is an arbitrary policy of \agentone.
\end{lemma}
\begin{proof}
The claim follows from Proposition 7 in \cite{syrgkanis2015fast} (more precisely, Theorem 19 and Lemma 20) by noting that the loss function is 
$\qvals_{\tixone}$ and that the update of $\polone_{\tixone}(\state)$ corresponds to OFTRL. That is, for each state $\state$ and $ l \in \{0, 1, ..., \floor{ \frac{\tix}{\Segment} } \}$:
\begin{align*}
\sum_{\tixone=l\cdot \Segment + 1}^{(l+1)\cdot \Segment}
\qvals_{\tixone}(\state) \cdot (\polone(\state)  -\polone_{\tixone}(\state))^\transpose &\le \frac{\Delta_{\regul}}{\ratelearn} + \ratelearn \cdot \sum_{\tixone=l\cdot \Segment + 1}^{(l+1)\cdot \Segment} \norm{ \qvals_{\tixone}(\state) -  \qvals_{(\tixone -1)\%\Segment}(\state)}_{\infty}^2 - \\
&- \frac{1}{4 \cdot \ratelearn} \cdot \sum_{\tixone=l\cdot \Segment + 1}^{(l+1)\cdot \Segment} \norm{\polone_{\tixone}(\state) - \polone_{(\tixone -1)\%\Segment}(\state)}_{1}^2,
\end{align*}
which implies the statement. (Note that for $\polone(s)$-difference we use $\norm{\cdot}_{1}$, whereas for $\polone$-difference we use $\norm{\cdot}_{\infty}$.)
\end{proof}

\begin{lemma}\label{lm_weight_diff_restart}
Let $ l \in \{0, 1, ..., \floor{ \frac{\tix}{\Segment} } \}$. The policy change magnitude within one segment ($\tixone \in \{l\cdot \Segment + 1, ..., (l+1)\cdot \Segment\}$) of \algexprestart~is bounded by:
\begin{align*} 
 \norm{\polone_{\tixone} - \polone_{(\tixone-1)\%\Segment}}_{\infty} \le \min \left  \{ 2,  \ratelearn \cdot  \frac{9}{1-e^{-\frac{1}{\mixtime}}} \right \}.
\end{align*}
Consequently:
\begin{align*}
\ratechangeone^{\Segment} \le  \min \left \{ 2,  \ratelearn \cdot  \frac{9}{1-e^{-\frac{1}{\mixtime}}} \right \}.
\end{align*}
\end{lemma}
\begin{proof}
The claim follows from Lemma 20 in \cite{syrgkanis2015fast} by noting that $\polone_{\tixone}(\state)$ are updated using OFTRL (see also Section 3.2 in  \cite{syrgkanis2015fast}), while $\qval$-values are bounded by $\frac{3}{1-e^{-\frac{1}{\mixtime}}}$ (see Lemma \ref{lm_q_bound}).  In particular from Lemma 20 in \cite{syrgkanis2015fast}, the triangle inequality, and Lemma \ref{lm_q_bound}, it follows: 
\begin{align*}
\norm{\polone_{\tixone}(\state) - \polone_{(\tixone-1)\%\Segment}(\state)}_{1} 
&\le \ratelearn \cdot \norm{\qvals_{(\tixone-1)\%\Segment}(\state) - \qvals_{\tixone}(\state)}_{\infty} + \ratelearn \cdot \norm{\qvals_{\tixone}(\state)}_{\infty} \le 2 \cdot \ratelearn \cdot \norm{\qvals_{\tixone}(\state)}_{\infty}  + \ratelearn \cdot \norm{\qvals_{(\tixone-1)\%\Segment}(\state)}_{\infty} \\
&\le  \ratelearn \cdot  \frac{9}{1-e^{-\frac{1}{\mixtime}}}.
\end{align*}
By taking into account that $\polone_{\tixone} \in \simplexone$, we know that $\norm{\polone_{\tixone}(\state) - \polone_{(\tixone-1)\%\Segment}(\state)}_{1} \le 2$, which together with the above proofs the first claim. The second claim follows by the definition. 
\end{proof}

\begin{theorem}\label{thm_regret_restart}
Let the learning rate of \algexprestart~be 
equal to $\ratelearn = \frac{1}{\Segment^{\frac{1}{4}}}$, and 
let $k > 0$ be such that $k \cdot \ratelearn > \ratechangeone^\Segment$ and $k \cdot \ratelearn > \ratechangetwo$.
Then, the regret of Algorithm \ref{algo_oftrl_rand_restart} is upper-bounded by: 
\begin{align*}
\regret(\tix) &\le  \left (\Delta_{\regul}  +  k^2 \cdot C_{\mixtime}^2 \right ) \cdot T \cdot \Segment^{-\frac{3}{4}} \nonumber + \frac{ 3 \cdot \influencetwo \cdot \ratechange_2}{(1-e^{-\frac{1}{\mixtime}})^2} \cdot \tix  \cdot \Segment.
\end{align*}
\end{theorem}

\begin{proof}
Lemma \ref{lm_eta_q_relation} provides the following useful identity:
\begin{align*}
\sum_{\tixone = 1}^{\tix} \avgrewt(\polone)-\avgrewt(\polonet) = \sum_{\tixone = 1}^{\tix} \statdist_{\polone, \poltwo_{\tixone}} \cdot ( \qvalst^\polone  - \qvalst^\polonet ).
\end{align*}
Furthermore, by the definition of $\qvalst^\polone$ and $\qvalst^\polonet$, we have:
\begin{align*}
\qvalst^\polone  -\qvalst^\polonet =\vecdot{\polone - \polonet, \qvals_{\tixone}}. 
\end{align*}
Using the two identities, gives:
\begin{align*}
\sum_{\tixone = 1}^{\tix} \avgrewt(\polone)-\avgrewt(\polonet)= \sum_{\tixone = 1}^{\tix} \statdist_{\polone, \poltwo_{\tixone}} \cdot \vecdot{\polone - \polonet, \qvals_{\tixone}}
&=  \sum_{\tixone = 1}^{\Segment} \statdist_{\polone, \poltwo_{\tixone}} \cdot \vecdot{\polone - \polonet, \qvals_{\tixone}} \\
&+\sum_{\tixone = \Segment + 1}^{2 \cdot \Segment} \statdist_{\polone, \poltwo_{\tixone}} \cdot \vecdot{\polone - \polonet, \qvals_{\tixone}}\\
&+ \cdots \\
&+\sum_{\tixone = \tix - \Segment + 1}^{\tix} \statdist_{\polone, \poltwo_{\tixone}} \cdot \vecdot{\polone - \polonet, \qvals_{\tixone}}.
\end{align*}
We will proceed by analyzing only the first summand, i.e., $\sum_{\tixone = 1}^{\Segment} \statdist_{\polone, \poltwo_{\tixone}} \cdot \vecdot{\polone - \polonet, \qvals_{\tixone}}$. The others can be analyzed analogously since experts with periodic restarts has the same behavior in all the segments $(1, ..., \Segment)$, $(\Segment + 1, ..., 2 \cdot \Segment)$, ...,  $(\tix - \Segment + 1, ..., \tix)$. We can rewrite the first summand as:  
\begin{align*}
 \sum_{\tixone = 1}^{\Segment} \statdist_{\polone, \poltwo_{\tixone}} \cdot \vecdot{\polone - \polonet, \qvals_{\tixone}} 
 = \sum_{\tixone = 1}^{\Segment} \statdist_{\polone, \poltwo_{1}} \cdot \vecdot{\polone - \polonet, \qvals_{\tixone}} +  \sum_{\tixone = 1}^{\Segment} (\statdist_{\polone, \poltwo_{\tixone}} - \statdist_{\polone, \poltwo_{1}}) \cdot \vecdot{\polone - \polonet, \qvals_{\tixone}},
\end{align*}
and proceed by bounding each of the obtained terms (summations) independently. Due to Lemma \ref{lm_oftrl_property_restart}, the first term is bounded by: 
\begin{align*}
&\sum_{\tixone = 1}^{\Segment} \statdist_{\polone, \poltwo_{1}} \cdot \vecdot{\polone - \polonet, \qvals_{\tixone}}
=
\statdist_{\polone, \poltwo_{1}} \cdot \sum_{\tixone = 1}^{\Segment} \vecdot{\polone - \polonet, \qvals_{\tixone}} \\
&\le 
\statdist_{\polone, \poltwo_{1}} \cdot \mathbf 1 \cdot  \left (  \frac{\Delta_{\regul}}{\ratelearn}  + \right . 
 \ratelearn \cdot \sum_{\tixone=1}^{\Segment}  \norm{ \qvals_{\tixone} -  \qvals_{(\tixone - 1)\%\Segment}}_{\max}^2  
\left . - \frac{1}{4 \cdot \ratelearn} \cdot \sum_{\tixone=1}^{\Segment}  \norm{\polone_{\tixone} - \polone_{(\tixone-1)\%\Segment}}_{\infty}^2  \right )\\
&\le
   \frac{\Delta_{\regul}}{\ratelearn}  + \ratelearn \cdot \sum_{\tixone=1}^{\Segment}  \norm{ \qvals_{\tixone} -  \qvals_{(\tixone - 1)\%\Segment}}_{\max}^2
\end{align*}
Proposition \ref{prop_C_bound} holds within segment $(1, ..., \Segment)$ (or any other segment), so Lemma \ref{lm_qvals_diff_t} and Lemma \ref{lm_weight_diff_restart} imply that the above is further bounded by:
\begin{align*}
\sum_{\tixone = 1}^{\Segment} \statdist_{\polone, \poltwo_{1}} \cdot \vecdot{\polone - \polonet, \qvals_{\tixone}} \le
 &\frac{\Delta_{\regul}}{\ratelearn}  + \ratelearn \cdot \sum_{\tixone=1}^{\Segment}  \norm{ \qvals_{\tixone} -  \qvals_{(\tixone - 1)\%\Segment}}_{\max}^2\\
&\rightarrow \textit{Using Lemma \ref{lm_qvals_diff_t}} \\ 
 &\le \frac{\Delta_{\regul}}{\ratelearn} + \Segment \cdot \ratelearn \cdot C_{\qval}^2   \\
 &\rightarrow \textit{By the definition of $C_{\mixtime}$} \\
 &\le\frac{\Delta_{\regul}}{\ratelearn} + \Segment \cdot \ratelearn \cdot C_{\mixtime}^2  \cdot \max \{ {\ratechangeone^{\Segment}}^2, \ratechangetwo^2 \}  \\
 &\rightarrow \textit{Using $k \cdot \ratelearn \ge \ratechangeone^\Segment$ and $k \cdot \ratelearn \ge \ratechangetwo$} \\
 &\le\frac{\Delta_{\regul}}{\ratelearn} + k^2 \cdot \Segment \cdot \ratelearn^3 \cdot C_{\mixtime}^2 
  \end{align*}
Notice that in the above analysis $\mathbf 1$ denoted a row vector of ones with $|\states|$ elements.
Finally, setting $\ratelearn = \frac{1}{\Segment^{\frac{1}{4}}}$ leads to the upper bound 
$(\Delta_{\regul} + k^2 \cdot C_{\mixtime}^2) \cdot \Segment^{\frac{1}{4}}$. 

The second term is bounded by its absolute value, which together with the triangle inequality and Holder's inequality, results in: 
\begin{align*}
&\sum_{\tixone =1}^{\Segment} ( \statdist_{\polone, \poltwo_{\tixone}} - \statdist_{\polone, \poltwo_{1}}) 
\cdot \vecdot{ \polone - \polone_{\tixone}, \qvals_{\tixone}}
\le \sum_{\tixone =1}^{\Segment} \norm{\statdist_{\polone, \poltwo_{\tixone}} - \statdist_{\polone, \poltwo_{1}}}_{1}  
 \cdot \norm{\vecdot{ \polone - \polone_{\tixone}, \qvals_{\tixone}} }_{\infty}.
\end{align*}
From Lemma \ref{lm_diff_pol_rew_general} and Lemma \ref{lm_q_bound} it follows that 
$\norm{\vecdot{ \polone - \polone_{\tixone}, \qvals_{\tixone}} }_{\infty} \le \frac{6}{1 - e^{-\frac{1}{\mixtime}}}$, which leads to the upper bound:  
\begin{align*}
\sum_{\tixone =1}^{\Segment} ( \statdist_{\polone, \poltwo_{\tixone}} - \statdist_{\polone, \poltwo_{1}}) 
\cdot \vecdot{ \polone - \polone_{\tixone}, \qvals_{\tixone}}
\le \sum_{\tixone =1}^{\Segment} \norm{\statdist_{\polone, \poltwo_{\tixone}} - \statdist_{\polone, \poltwo_{1}}}_{1} \cdot \frac{6}{1 - e^{-\frac{1}{\mixtime}}},
\end{align*}
which by identity:
\begin{align*}
\statdist_{\polone, \poltwo_{\tixone}} - \statdist_{\polone, \poltwo_{1}} &= \statdist_{\polone, \poltwo_{\tixone}} - \statdist_{\polone, \poltwo_{\tixone-1}} \\
&+ \statdist_{\polone, \poltwo_{\tixone-2}} - \statdist_{\polone, \poltwo_{\tixone - 3}} \\
&...\\
&+ \statdist_{\polone, \poltwo_{3}} -  \statdist_{\polone, \poltwo_{2}}\\
&+ \statdist_{\polone, \poltwo_{2}} -  \statdist_{\polone, \poltwo_{1}}\\
&= \sum_{k = 1}^{\tixone-1} (\statdist_{\polone, \poltwo_{k+1}} - \statdist_{\polone, \poltwo_{k}}),
\end{align*}
and the triangle inequality is bounded by:
\begin{align*}
\sum_{\tixone =1}^{\Segment} ( \statdist_{\polone, \poltwo_{\tixone}} - \statdist_{\polone, \poltwo_{1}}) 
\cdot \vecdot{ \polone - \polone_{\tixone}, \qvals_{\tixone}}
\le & \frac{6}{1 - e^{-\frac{1}{\mixtime}}} \cdot \sum_{\tixone=1}^{\Segment} \sum_{k = 1}^{\tixone-1} \norm{\statdist_{\polone, \poltwo_{k+1}} - \statdist_{\polone, \poltwo_{k}}}_{1}.
\end{align*}
Application of Lemma \ref{lm_stat_dist_inequality} additionally bounds the last term by: 
\begin{align*}
\frac{6}{1 - e^{-\frac{1}{\mixtime}}} \cdot \sum_{\tixone =1}^{\Segment} \sum_{k = 1}^{\tixone-1}  \frac{\influencetwo \cdot \ratechange_2}{1-e^{-\frac{1}{\mixtime}}} 
&=  \frac{6}{1 - e^{-\frac{1}{\mixtime}}} \cdot \frac{\influencetwo \cdot \ratechange_2}{1-e^{-\frac{1}{\mixtime}}} \cdot   \sum_{\tixone =1}^{\Segment} (t-1)  \\
&\le  \frac{6}{1 - e^{-\frac{1}{\mixtime}}} \cdot \frac{\influencetwo \cdot \ratechange_2}{1-e^{-\frac{1}{\mixtime}}} \cdot   \frac{\Segment^2}{2}  \\
&=  \frac{3}{1 - e^{-\frac{1}{\mixtime}}} \cdot \frac{\influencetwo \cdot \ratechange_2}{1-e^{-\frac{1}{\mixtime}}} \cdot  \Segment^2.
\end{align*}
Therefore, $\sum_{\tixone = 1}^{\Segment} \statdist_{\polone, \poltwo_{\tixone}} \cdot \vecdot{\polone - \polonet, \qvals_{\tixone}}$ is bounded by:
\begin{align*}
\sum_{\tixone = 1}^{\Segment} \statdist_{\polone, \poltwo_{\tixone}} \cdot \vecdot{\polone - \polonet, \qvals_{\tixone}} \le \left (\Delta_{\regul} + k^2\cdot C_{\mixtime}^2 \right ) \cdot \Segment^{\frac{1}{4}} + \frac{3}{1 - e^{-\frac{1}{\mixtime}}} \cdot \frac{\influencetwo \cdot \ratechange_2}{1-e^{-\frac{1}{\mixtime}}} \cdot  \Segment^2.
\end{align*}

By including all the other segments, in total $\frac{T}{\Segment}$ of them (at most), we obtain: 
\begin{align*}
&\regret(\tix) = \sum_{\tixone = 1}^{\tix} \avgrewt(\polone)-\avgrewt(\polonet)  
\le \left (\Delta_{\regul}  +  k^2 \cdot C_{\mixtime}^2 \right ) \cdot T \cdot \Segment^{-\frac{3}{4}} + \frac{ 3 \cdot \influencetwo \cdot \ratechange_2}{(1-e^{-\frac{1}{\mixtime}})^2} \cdot \tix  \cdot \Segment.
\end{align*}
\end{proof}

Given that the bound in Theorem \ref{thm_regret_restart} is (order-wise) the same as the bound in Theorem \ref{thm_regret}, the statement of the main result of the paper (Theorem \ref{thm_main}) holds for experts with periodic restarts as well. Namely, one can easily adjust the proof of Theorem \ref{thm_main} so that the statement of the theorem holds for experts with periodic restarts (e.g., by utilizing Lemma \ref{lm_weight_diff_restart} instead of Lemma \ref{lm_weight_diff}).

\clearpage
\section{Proof of Theorem \ref{thm_hardness}}\label{appendix.proof_thm_hardness}

\begin{proof}
To obtain the hardness result, it suffices to reduce the adversarial online shortest path problem of \cite{abbasi2013online} to our setting in polynomial time  --- namely, 
the adversarial online shortest part problem is at least as hard as online agnostic parity (see Theorem 5 in \cite{abbasi2013online}), so the proposed reduction would yield the result. 

In the online shortest path problem, at each round $\tixone$, an agent has to select a path from a start node to an end node 
in a direct acyclic graph $g_{\tixone}$. The graph $g_{\tixone}$ is characterized by a start node, an end node, and $L$ layers that contain 
other nodes with a restriction that a node in layer $l$ only connects to nodes in $l-1$ (or start node if $l = 1$) and nodes in $l+1$
(or the end node if $l = L$). The reward $\rew_{g_\tixone}$ (alternatively, loss) is revealed after the decision maker chooses a path in $g_{\tixone}$.
Note that the proof of Theorem 5 in \cite{abbasi2013online}, which shows the computational hardness of the problem, uses only binary rewards and $2$ nodes per layer, 
each node having at least one successor (except the end node). Thus, to prove our theorem, it suffices to reduce the simpler version of the
online shortest problem to our problem in polynomial time. The simpler shortest part problem is characterized by: 
\begin{itemize}
\item A set of nodes $N_{g_\tixone}$, enumerated by $n_{0}$, $n_{1,0}$, $n_{1,1}$, ..., $n_{l,0}$, $n_{l,1}$, ..., $n_{L, 0}$, $n_{L, 1}$, $n_{L+1}$, where 
$n_{0}$ is the start node, while $n_{L+1}$ is the end node. 
\item A time dependent set of directed edges $E_{g_\tixone}$ that define connections between nodes from layer $l$ and nodes in layer $l+1$, where only directed edges of 
type $(n_{0}, n_{1,0})$, $(n_{0}, n_{1,1})$, ..., $(n_{l,0}, n_{l+1,0})$, $(n_{l,1}, n_{l+1,0})$, $(n_{l,0}, n_{l+1,1})$, $(n_{l,1}, n_{l+1,1})$, ...,
$(n_{L,0}, n_{L+1})$, $(n_{L,1}, n_{L+1})$, are allowed. We impose condition that each node $n \ne n_{L+1}$ has an outgoing edge. 
\item A time dependent reward function $\rew_{g_\tixone}$ that takes values in $\{0, 1\}$ and whose value $\rew_{g_\tixone}((n, n'))$ determines the weight of edge $(n, n')$. 
\item A time dependent policy $\pi_{g,\tixone}$ that defines a successor node for each node other than $n_{L+1}$, e.g., 
$\pi_{g,\tixone}(n_{l,1}) = n_{l+1,0}$. If node $\pi_{g,\tixone}(n_{l,1}) \notin E_{g_\tixone}$, then we take an arbitrary (predefined) successor node of 
$n_{l, 1}$ that is in $E_{g_\tixone}$. We can define a value of each policy $\pi_{g,\tixone}$ as:
\begin{align*} 
V_t(\pi_{g,\tixone}) = \frac{1}{L} \cdot \sum_{l = 1}^{L} \rew_{g_\tixone}(\pi_{g,\tixone}^{l}(n_0), \pi_{g,\tixone}^{l + 1}(n_0)),
\end{align*}
where $\pi_{g,\tixone}^{l}(n_0)$ is defined inductively as $\pi_{g,\tixone}^{l}(n_0) = \pi_{g,\tixone}(\pi_{g,\tixone}^{l-1}(n_0))$ with $\pi_{g,\tixone}^{0}(n_0) = n_0$.
\item The objective is to be competitive w.r.t. any stationary policy $\pi$, that is, to minimize the regret $\sum_{\tixone = 1}^{\tix} (V_t(\pi) - V_t(\pi_{g,\tixone}))$.
\end{itemize}

{\em Reduction:} 
~\\
~\\
\textbf{Defining \mdp:} Given a graph $g_{\tixone}$, let us define a corresponding \mdp~ in our setting:
\begin{itemize}
\item We associate each node in  $g_{\tixone}$ with a state in \mdp:
state $\state_0$ corresponds to the start node $n_0$; state $\state_{L+1}$ corresponds to the end node $n_{L+1}$; and for the nodes $n_{l, 0}$ and $n_{l,1}$ in layer $l$, we define states 
$\state_{l, 0}$ and $\state_{l, 1}$ respectively. Therefore, $\states = \{ \state_0, ..., \state_{l, 0}, \state_{l, 1}, ..., \state_{L+1} \}$.
\item The action space for \agentone~ describes possible choices in the online shortest path --- $\actionsone = \{0, 1\}$ corresponding to the choice of a next node $n_{l, 0}$ or $n_{l, 1}$ from nodes $n_{l-1, 0}$ or $n_{l-1, 1}$. The action space for \agenttwo~ is equal to $\actionstwo = \{\actiontwo_{a}, \actiontwo_{0,0}, \actiontwo_{0,1}, \actiontwo_{1,0}, , \actiontwo_{1,1}, \actiontwo_{b, 0}, \actiontwo_{b, 1}, \actiontwo_{b, 2}, \actiontwo_{b, 3}\}$.
\item All the transitions are deterministic given $\actionone$ and $\actiontwo$. 
The transitions from node 
$\state_{l, 0}$ and $\state_{l, 1}$ (for $l < L$) are primarily determined by the action of \agenttwo: 
\begin{itemize}
\item if it is equal to $\actiontwo_{a}$, the next state remains the same, 
\item if it is $\actiontwo_{0, x}$, 
the next state is $\state_{l+1,0}$, 
\item if it is $\actiontwo_{1, x}$ then the next state is $\state_{l+1, 1}$,  
\item if it is $\actiontwo_{b, x}$ the next state is determined by \agentone's action, i.e., the next state is $\state_{l+1, \actionone}$. 
\end{itemize}
That is, \agentone~ only affects the transition if \agenttwo~ selects 
$\actiontwo_{b,x}$. Analogously we define transitions for $\state_0$, having in mind that next layer is layer $1$. 
From $\state_{L, 0}$ and  $\state_{L, 1}$ the transitions 
lead to $\state_{L+1}$, unless $\actiontwo$ is equal to $\actiontwo_{a}$, in which case the the state remains the same.
From $\state_{L+1}$, all the transitions lead to $\state_0$. Put together:
\begin{align*}
\trans(\state,\actionone, \actionone, \state_{new}) = 1 \begin{cases}
\mbox{ if } \state = \state_{0} \text{ and } \state_{new} = \state_{0} \text{ and } \actiontwo = \actiontwo_a  \\
\mbox{ if } \state = \state_{0} \text{ and } \state_{new} = \state_{1, 0} \text{ and } \actiontwo = \actiontwo_{0, x}  \\
\mbox{ if } \state = \state_{0} \text{ and } \state_{new} = \state_{1, 1} \text{ and } \actiontwo = \actiontwo_{1, x} \\
\mbox{ if } \state = \state_{0} \text{ and } \state_{new} = \state_{1, \actionone} \text{ and } \actiontwo = \actiontwo_{b, x}  \\
\mbox{ if } \state = \state_{l, x} \text{ and } \state_{new} = \state_{l, x} \text{ and } \actiontwo = \actiontwo_a  \text{ and } 0 < l < L \\
\mbox{ if } \state = \state_{l, x} \text{ and } \state_{new} = \state_{l+1, 0} \text{ and } \actiontwo = \actiontwo_{0, x}  \text{ and } 0 < l < L \\
\mbox{ if } \state = \state_{l, x} \text{ and } \state_{new} = \state_{l+1, 1} \text{ and } \actiontwo = \actiontwo_{1, x}  \text{ and } 0 < l < L \\
\mbox{ if } \state = \state_{l, x} \text{ and } \state_{new} = \state_{l+1, \actionone} \text{ and } \actiontwo = \actiontwo_{b, x}  \text{ and } 0 < l < L \\
\mbox{ if } \state = \state_{L, x} \text{ and } \state_{new} = \state_{L, x} \text{ and } \actiontwo = \actiontwo_a \\
\mbox{ if } \state = \state_{L, x} \text{ and } \state_{new} = \state_{L + 1} \text{ and } \actiontwo \ne \actiontwo_a \\
\mbox{ if } \state = \state_{L+1} \text{ and } \state_{new} = \state_{0}
\end{cases},
\end{align*}
and it is $0$ otherwise. 
\item For the reward function, we set:
\begin{align*}
&\rew(\state, \actionone, \actiontwo_{a}) = 0,\\
&\rew(\state, \actionone, \actiontwo_{0, 0}) = 0, \\
&\rew(\state, \actionone, \actiontwo_{0, 1}) = 1, \\
&\rew(\state, \actionone, \actiontwo_{1, 0}) = 0, \\
&\rew(\state, \actionone, \actiontwo_{1, 1}) = 1, \\
&\rew(\state, \actionone, \actiontwo_{b, 0}) = 0,\\
&\rew(\state, \actionone, \actiontwo_{b, 1}) = \mathbbm 1_{\actionone = 0},\\ 
&\rew(\state, \actionone, \actiontwo_{b, 2}) = \mathbbm 1_{\actionone = 1},\\
&\rew(\state, \actionone, \actiontwo_{b, 3}) = 1.
\end{align*}
This implies that  \agenttwo~ can control the reward of each $(\state,\actionone)$ pair by an appropriate action selection. In other, we can simulate the weights (rewards) of graph $g_{\tixone}$ by an appropriate choice of \agenttwo's policy.  
\end{itemize}
~\\
~\\
\textbf{Defining policy $\poltwot$:}
Let us now define the policy of \agenttwo~ using graph $g_\tixone$ --- note $\rew_{g_\tixone}$ is not revealed to \agentone~ before episode $\tixone$. 
Let us consider the case $0 < l \le L$.
First, we set $\poltwo_\tixone(\state_{l, x}, \actiontwo_{a}) = 1 - \ratechangetwo$, to ensure 
that policy change of \agenttwo~ is bounded by $\ratechangetwo$. For the other $(\state_{l, x}, \actiontwo)$ pairs, we set $\poltwo_\tixone(\state_{l, x}, \actiontwo) = 0$
except:
\begin{itemize}
\item if node $n_{l, x}$ has only $n_{l+1, i}$ as a successor and $\rew_{g_\tixone}(n_{l, x}, n_{l+1, i}) = j$, then $\poltwo_\tixone(\state_{l, x}, \actiontwo_{i, j}) = \ratechangetwo$;
\item if node $n_{l, x}$ has both  $n_{l+1, 0}$ and $n_{l+1, 1}$ as successors, then $\poltwo_\tixone(\state_{l, x}, \actiontwo_{b, r}) = \ratechangetwo$, where
$r = \rew_{g_\tixone}(n_{l, x}, n_{l+1, 0}) + 2 \cdot \rew_{g_\tixone}(n_{l, x}, n_{l+1, 1})$.
\end{itemize}
The first point ensures that when $n_{l, x}$ has only one successor, $\poltwot$ will allow transitions from $\state_{l,x}$ only to the corresponding successor as $\actiontwo$ will either be $\actiontwo_{0, j}$ or $\actiontwo_{1,j}$. The choice of $j$ in $\actiontwo_{., j}$ ensures that the reward for that transition is equal to the corresponding weight in graph $g_\tixone$ since for $\actiontwo_{i, 0}$, the reward is $0$ and for $\actiontwo_{i, 1}$, the reward is $1$. The second point ensures that when $n_{l, x}$ has both successors, the transition is dependent on \agentone's action. This is true because \agenttwo's action is $\actiontwo_{b, r}$, which defines transitions according to \agentone's action $\actionone$. Notice that when a node has  two possible successors in $g_\tixone$, $r$ encodes the associated weight (reward) values of the node's output edges: $r = 0$ - weights are 0, $r = 1$ - the weight associate to transition $n_{l,x} \rightarrow n_{l+1,0}$ is $1$ and the other weight (associated with transition $n_{l,x} \rightarrow n_{l+1,1}$) is $0$, $r = 2$ - the other way around,  $r = 3$ - both weights are $1$. So $r$ in $\actiontwo_{b, r}$ ensures that $\state_{l,x}$ has the same reward profile  as node $n_{l,x}$ has on the corresponding transitions .  
Therefore, the inner layers of \mdp~ are properly reflecting inner layers of $g_\tixone$, except that we remain in each state with probability $1-\ratechangetwo$ irrespective of \agentone's policy. 
Analogously we define $\poltwo_\tixone$ for $\state_0$, having in mind that the next later is layer $1$, so the same holds for $\state_0$. The choice of $\poltwo_\tixone(\state_{L+1})$ is irrelevant for both transitions and rewards, so we can set it to $\poltwo_\tixone(\state_{L+1}, \actiontwo_{a}) = 1$.  
To sumarize:
\begin{align*}
&\poltwo_\tixone(\state, \actiontwo) = 1 - \ratechangetwo 
\begin{cases}
\mbox{ if } \state = \state_{0} \text{ and } \text{ and } \actiontwo = \actiontwo_a  \\
\mbox{ if } \state = \state_{l,x} \text{ and } \text{ and } \actiontwo = \actiontwo_a   \text{ and } 0 < l \le L  \\
\end{cases}\\
\\
&\poltwo_\tixone(\state, \actiontwo) = \ratechangetwo 
\begin{cases}
\mbox{ if } \state = \state_{0} \text{ and } (n_{0}, n_{1, 0}) \in E_{g_\tixone} \text{ and } (n_{0}, n_{1, 1}) \notin E_{g_\tixone} \text{ and } \actiontwo = \actiontwo_{0, \rew_{g_\tixone}(n_{0}, n_{1, 0})}  \\
\mbox{ if } \state = \state_{0} \text{ and } (n_{0}, n_{1, 0}) \notin E_{g_\tixone} \text{ and } (n_{0}, n_{1, 1}) \in E_{g_\tixone} \text{ and } \actiontwo = \actiontwo_{1, \rew_{g_\tixone}(n_{0}, n_{1, 1})}  \\
\mbox{ if } \state = \state_{0} \text{ and } (n_{0}, n_{1, 0}) \in E_{g_\tixone} \text{ and } (n_{0}, n_{1, 1}) \in E_{g_\tixone} \text{ and } \actiontwo = \actiontwo_{b, r} \\
\mbox{ if } \state = \state_{l, x} \text{ and } (n_{l, x}, n_{l+1, 0}) \in E_{g_\tixone} \text{ and } (n_{l, x}, n_{l+1, 1}) \notin E_{g_\tixone} \text{ and } \actiontwo = \actiontwo_{0, \rew_{g_\tixone}(n_{l, x}, n_{l+1, 0})}  \\
\mbox{ if } \state = \state_{l, x} \text{ and } (n_{l, x}, n_{l+1, 0}) \notin E_{g_\tixone} \text{ and } (n_{l, x}, n_{l+1, 1}) \in E_{g_\tixone} \text{ and } \actiontwo = \actiontwo_{1, \rew_{g_\tixone}(n_{l, x}, n_{l+1, 1})}  \\
\mbox{ if } \state = \state_{l, x} \text{ and } (n_{l, x}, n_{l+1, 0}) \in E_{g_\tixone} \text{ and } (n_{l, x}, n_{l+1, 1}) \in E_{g_\tixone} \text{ and } \actiontwo = \actiontwo_{b, r}
\end{cases}\\
\\
&\poltwo_\tixone(\state_{L+1}, \actiontwo_{a}) = 1,
\end{align*}
and $\poltwo_\tixone(\state, \actiontwo) = 0$ otherwise, where $0 < l \le L$ and $r  = \rew_{g_\tixone}(n_{l, x}, n_{l+1, 0}) + 2 \cdot \rew_{g_\tixone}(n_{l, x}, n_{l+1, 1})$. 
By these choices of \agenttwo's policy we have encoded the structure of $g_\tixone$, with the exception that in each state $\state \ne \state_{L+1}$ we remain with probability $1-\ratechangetwo$.

Finally, we ought to encode any policy $\pi_{g,\tixone}$ with \agentone's policy. Let us define $\polone_\tixone$ as:
\begin{align*}
\polone_\tixone(\state_0) &= \mathbbm 1_{\bm \pi_{g,\tixone}(n_0) = n_{1,1}}\\
\polone_\tixone(\state_{l,x}) &= \mathbbm 1_{\bm \pi_{g,\tixone}(n_{l,x}) = n_{l+1,1}}\\
\polone_\tixone(\state_{L,x}) &= \polone_\tixone(\state_{L+1}) = 0
\end{align*}
where $0 < l < L$. In other words, the action of choosing node $n_{x, 0}$ in the shortest path problem, corresponds to action $0$, whereas the action of choosing node $n_{x, 1}$ in the shortest path problem, corresponds to action $1$. Notice that \agentone's actions in states $\state_{L,x}$ and $\state_{L+1}$ are not not important as they don't affect the agent's rewards nor transitions.  
Given that $\polone_\tixone$ represents $\pi_{g,\tixone}$, whereas with $\poltwot$ we have encoded the relevant details of $g_\tixone$ in \mdp,
it follows that the the average reward obtained for episode $\tixone$ in the \mdp~is equal to:
\begin{align*}
\avgrewt(\polone_\tixone) = \Thetabound \left ( \ratechangetwo \right ) \cdot V_\tixone (\pi_{g,\tixone}).
\end{align*}
Coefficient $\ratechangetwo$ comes from the fact that the number of rounds to reach state $\state_{L+1}$ from $\state_{L}$ is expected to be $\frac{1}{\ratechangetwo} \cdot (L+1)$ (i.e., in each state the agents are expected to be $\frac{1}{\ratechangetwo}$ rounds given that the probability of remaining in a state other than $\state_{L+1}$ is $1-\ratechangetwo$), whereas it takes $L+1$ steps 
to reach the end node in the shortest path problem. 
Notice that $ \Thetabound \left (\ratechangetwo \right )$ is not dependent on time horizon $\tix$, while the reduction we described is efficient (polynomial in the problem parameters). 
Therefore, if we can find an efficient no-regret algorithm for the setting of this paper, then we can also find an efficient no-regret learning algorithm for
the adversarial online shortest path problem, which implies the statement of the theorem.
\end{proof}

}

}{}

\end{document}